\documentclass{article}

\usepackage[final]{neurips_2025}

\usepackage{algorithm}
\usepackage{algpseudocode}

\usepackage{amsmath}
\usepackage{amssymb}
\usepackage{mathtools}
\usepackage{amsthm}
\usepackage{times}
\usepackage{caption}

\usepackage{footnote}

\usepackage{amsmath,amsfonts,bm}

\def\1{\bm{1}}

\DeclareMathAlphabet{\mathsfit}{\encodingdefault}{\sfdefault}{m}{sl}
\SetMathAlphabet{\mathsfit}{bold}{\encodingdefault}{\sfdefault}{bx}{n}

\newcommand{\E}{\mathbb{E}}

\newcommand{\R}{\mathbb{R}}

\DeclareMathOperator{\Tr}{Tr}

\usepackage{floatflt}

\usepackage{hyperref}
\usepackage{url}

\usepackage{bm}
\usepackage{amsmath}
\usepackage{amssymb}
\usepackage{mathtools}
\usepackage{amsthm}
\usepackage{wrapfig}
\usepackage[symbol]{footmisc}

\theoremstyle{plain}
\newtheorem{theorem}{Theorem}[section]

\newtheorem{proposition}[theorem]{Proposition}

\theoremstyle{definition}

\theoremstyle{remark}

\newcommand{\bx}{\mathbf{x}}

\newcommand{\cM}{\mathcal{M}}
\newcommand{\p}{\partial}
\newcommand{\SO}{\text{SO}}

\newcommand{\GL}{\text{GL}}
\newcommand{\cO}{\mathcal{O}}
\newcommand{\rank}{\text{rank }}

\newcommand{\ba}{\mathbf{a}}
\newcommand{\be}{\mathbf{e}}
\newcommand{\by}{\mathbf{y}}
\newcommand{\bbf}{\mathbf{f}}
\newcommand{\bW}{\mathbf{W}}
\newcommand{\tbx}{\widetilde{\bx}}
\newcommand{\tx}{\widetilde{x}}
\newcommand{\ty}{\widetilde{y}}

\newcommand{\bmu}{\mathbf{\boldsymbol{\mu}}}
\newcommand{\bs}{\mathbf{s}}
\newcommand{\hx}{\widehat{x}}
\newcommand{\bbeta}{\boldsymbol{\eta}}

\newcommand{\hz}{\widehat{z}}
\newcommand{\hbx}{\widehat{\bx}}
\newcommand{\bA}{\mathbf{A}}

\newcommand{\cN}{\mathcal{N}}
\newcommand{\bSigma}{\mathbf{\Sigma}}

\newcommand{\SE}{\text{SE}}
\newcommand{\g}{\mathfrak{g}}
\newcommand{\so}{\mathfrak{so}}

\newcommand{\cL}{\mathcal{L}}
\newcommand{\bcL}{\boldsymbol{\cL}}
\newcommand{\btau}{\boldsymbol{\tau}}
\newcommand{\btheta}{\boldsymbol{\theta}}
\newcommand{\bPi}{\mathbf{\Pi}}

\newcommand{\grad}{\nabla}

\newcommand{\bI}{\mathbf{I}}

\newcommand{\bzero}{\mathbf{0}}

\usepackage[utf8]{inputenc} 
\usepackage[T1]{fontenc}    
\usepackage{hyperref}       
\usepackage{url}            
\usepackage{booktabs}       
\usepackage{amsfonts}       
\usepackage{nicefrac}       
\usepackage{microtype}      
\usepackage{xcolor}         

\title{Diffusion Generative Modeling on Lie Group Representations}

\author{Marco Bertolini\thanks{Shared first authorship.}\ ,  Tuan Le\footnotemark[1] \ \& Djork-Arn\'e Clevert\\
Machine Learning Research\\ 
Pfizer Worldwide Research and Development \\
Friedrichstraße 110, 10117 Berlin, Germany\\
\texttt{\{marco.bertolini,tuan.le,djork-arne.clevert\}@pfizer.com}
}
\begin{document}

\maketitle

\begin{abstract}
We introduce a novel class of score-based diffusion processes that operate directly in the representation space of Lie groups.
Leveraging the framework of Generalized Score Matching, we derive a class of Langevin dynamics that decomposes as a direct sum of Lie algebra representations,
enabling the modeling of any target distribution on any (non-Abelian) Lie group.
Standard score-matching emerges as a special case of our framework when the Lie group is the translation group. 
We prove that our generalized generative processes arise as solutions to a new class of paired stochastic differential equations (SDEs), introduced here for the first time. 
We validate our approach through experiments on diverse data types, demonstrating its effectiveness in real-world applications such as $\text{SO}(3)$-guided molecular conformer generation and modeling ligand-specific global $\text{SE}(3)$ transformations for molecular docking, showing improvement in comparison to Riemannian diffusion on the group itself.
We show that an appropriate choice of Lie group enhances learning efficiency by reducing the effective dimensionality of the trajectory space and enables the modeling of transitions between complex data distributions.
\end{abstract}

\section{Introduction}\label{sec:introduction}

Deep probabilistic generative modeling amounts to creating data from a known tractable prior distribution. 
Score-based models \citep{hyvarinen2005estimation, sohl2015deep, ho2020denoising, huang2021variational, song2021maximum, song2020score} achieve this by learning to reverse a corruption process of the data.
Most algorithms assume an Euclidean data space $X$, yet many scientific applications \citep{brehmer2020flows, zhang2024diffpack, klimovskaia2020poincare, karpatne2018machine} involve distributions on curved manifolds $\mathcal{M}$. While significant progress has been made in developing the theory of diffusion in curved spaces \citep{de2022riemannian, huang2022riemannian}, key challenges remain: parametrizing vector fields on general $\cM$ is unsolved, and Langevin updates require projection to preserve the manifold structure.  
Even when $\mathcal{M} = G$ is a Lie group, denoising score-matching remains a challenge for general non-Abelian groups, thus necessitating explicit trajectory simulation. Recent findings \citep{abramson2024accurate} highlight this complexity, as diffusion was performed in raw Cartesian coordinates rather than explicitly modeling the torsion space, given its representational difficulty and lack of performance gain.  

An appropriate representation that leverages the symmetry property of the data should, however, enable models to better capture the underlying physical laws.
The limited performance of manifold-based diffusion must thus stem from technical and computational difficulties rather than fundamental principles. This work seeks to 
reconcile this expectation with the empirical findings
by addressing the question:  
\textit{Given a Lie group $G$ acting on Euclidean space $X$ through a map (representation)} 
\begin{wrapfigure}{r}{0.5\columnwidth}
\begin{center}
\includegraphics[width=0.5\columnwidth]{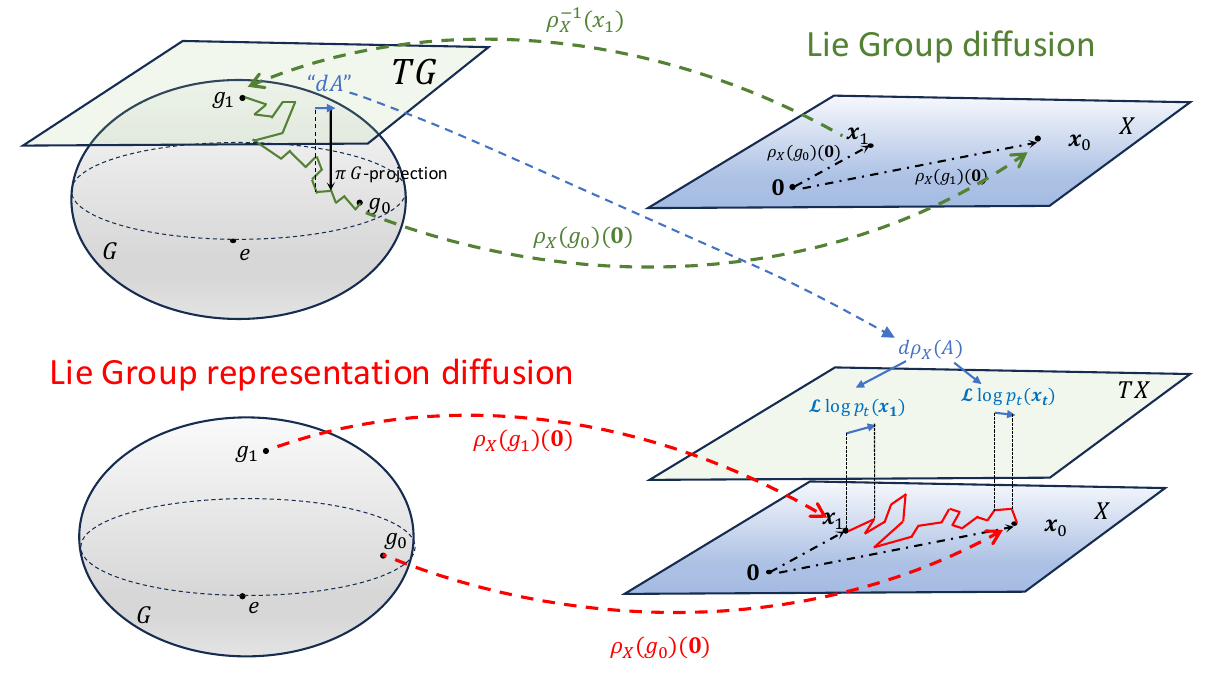}
\end{center}
\caption{
Comparison of strategies between Lie group (top) diffusion and our proposed Lie group representation diffusion (bottom).}
\label{fig:trajectories}
\end{wrapfigure}
\textit{$\rho_X:G\rightarrow \GL(X)$, can we construct a \textbf{generative process on $X$ that models any distribution on $G$}, thus retaining the advantages of flat-space diffusion while capturing non-trivial manifold structures?}
We address this question by constructing the diffusion process directly in the representation space, defined as the image of the group action map, $\mathrm{Im}(\rho_X) \subseteq \mathrm{GL}(X)$. This yields a matrix-valued diffusion process in $\mathrm{GL}(X)$ which, when applied to elements of $X$, induces a stochastic flow corresponding to infinitesimal Lie group transformations, i.e., Lie algebra elements. In this way, the process preserves the geometric inductive bias of the (curved) Lie group while remaining entirely within the flat vector space $X$.
Our construction builds on the framework of \emph{Generalized Score Matching} (GSM)~\citep{lyu2009interpretation, lin2016estimation}, which estimates probability densities via the generalized score function $\mathcal{L} \log p(\bx)$ for a suitable linear operator $\mathcal{L}$. We show that the $G$-induced generative process satisfies a continuous-time stochastic differential equation (SDE) involving this generalized score.
As illustrated in Figure~\ref{fig:trajectories}, our approach differs to diffusion processes directly on the Lie group: rather than mapping data to the group and back via the representation map, we remain in $X$ throughout, using the differential of the representation $d\rho_X: TG \rightarrow TX$ to guide the Langevin dynamics.

In short, \textit{we propose an exact SDE-based diffusion framework that enables general generative modeling on Lie group representations}, thus combining the advantages of curved dynamics with the theoretical and practical effectiveness of Euclidean diffusion. 
Our method realizes \textbf{simulation-free} training of Lie group-like diffusion models, and it provides a novel approach to denoising score-matching for general non-Abelian groups.
Our main contributions can be summarized as follows:

\textbf{Generalized score matching via Lie algebras:} We 
extend generalized score matching on $X$ to estimate the score of any distribution on a Lie group $G$ acting on $X$.
We elucidate the conditions for a suitable $G$ (valid for any differentiable manifold $X$).
We recover standard score-matching as a specific case of our framework, corresponding to the group $G = T(n)$ of translations on $X=\R^n$.

\textbf{Lie group representation diffusion processes as exact solution of a novel class of SDEs:} We introduce a \textit{new class of solvable} SDEs that govern Lie group diffusion via Euclidean coordinates, significantly expanding the range of processes that can be addressed using score-based modeling techniques. We also show that our approach extends naturally to flow matching (Appendix \ref{app:flow_matching}).

\textbf{Dimensionality reduction, bridging non-trivial distributions and trajectory disentanglement:} 
Through extensive experiments, we demonstrate that: (1) our approach can estimate, regardless of the choice of $G$, any probability density (Sections \ref{ss:toys} (2,3,4d distributions) and \ref{ss:QM9} (QM9); (2) by appropriately selecting $G$ to align with the data structure, the learning process is significantly simplified, effectively reducing its dimensionality (Section \ref{ss:MNIST}(MNIST)) (3) our framework enables solutions to processes that are challenging or unfeasible with standard score matching, such as bridging between complex data-driven distributions (Section \ref{ss:MNIST} (MNIST) and \ref{ss:crossdocked} (CrossDocked)).

\section{Diffusion dynamics through Lie algebras}

We start this section by setting up notation and review the connection between vector fields and Lie algebra actions on manifolds. 
A Lie group $G$ is a group that is also a finite-dimensional differentiable manifold, such that the group operations of multiplication $\cdot:G\times G \rightarrow G$ and inversion are $C^\infty$-functions\footnotemark[2]\footnotetext[2]{We restrict ourselves to real Lie groups.  
It would be interesting to extend our analysis to the complex case \citep{le2021parameterized}.}. 
A Lie algebra $\g$ is a vector space equipped with an operation, the Lie bracket, $[,]:\g \times\g \rightarrow\g$,
satisfying the Jacobi identity. 
Every Lie group gives rise to a Lie algebra as its tangent space at the identity, $\g = T_eG$, and the Lie bracket is the commutator of tangent vectors, $[A,B] = AB - BA$.
In this work, we are interested in how Lie groups and Lie algebras act on spaces. 
Given a manifold $X$, a (left) \textbf{group action} of $G$ on $X$ is an associative map\footnote{In the manuscript we adopt both notations $\rho_X(g)(\bx)$, derived from defining $\rho_X:G\rightarrow \GL(X)$, and $\rho_X(g,\bx)$, derived from the definition $\rho_X:G\times X\rightarrow X$, which are obviously equivalent.} $\rho_X: G \times X \rightarrow X$ such that 
$\rho_X(e,x)=x, \forall x\in X$.
Fundamental concepts associated with a group action are the ones of orbits and stabilizers. The \textbf{orbit} of $\bx\in X$ is the set of elements in $X$ which can be reached from $\bx$ through the action of $G$, i.e., $G\cdot \bx = \{ \rho_X(g)(\bx),~ g\in G \}$. The \textbf{stabilizer} subgroup of $G$ with respect to $\bx$ is the set of group elements that fix $\bx$, $G_{\bx} = \{ g\in G | \rho_X(g)(\bx) = \bx \}$.
The \textbf{action of a Lie algebra} on $X$,
$\mathfrak{A}: \g \rightarrow \text{Vect}(X)$ is a Lie algebra homomorphism and maps elements of $\g$ to vector fields on $X$ such that the map $\g \times X \rightarrow TX, (A, \bx) \mapsto \mathfrak{A}(A)(\bx)$ is smooth. 
\begin{figure}[t]
    \centering
    \includegraphics[width=\textwidth]{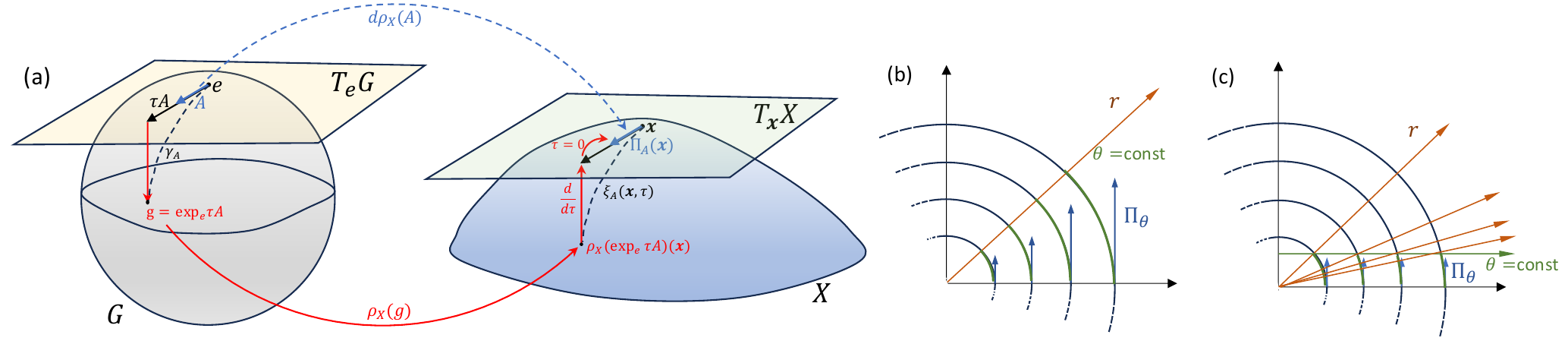}
    \caption{
    (a) Depiction of the fundamental vector field definition \eqref{eq:fund_vectorfields}.
    Flow coordinates for a pair of commuting (b) and not-commuting ones vector fields (c).}
    \label{fig:fundamental}
    \vspace{-0.1cm}
\end{figure}
Given $A\in\g$ and a group action $\rho_X$, the \textbf{flow} on $X$ induced by $\rho_X$ is given by $\xi_A: X \times \R \rightarrow X, (\bx, \tau) \rightarrow \rho_X\left(\exp(\tau A)\right)(\bx)$, where the map $\exp:\g \rightarrow G$ is defined by $\exp(A) = \gamma_A(1)$, where $\gamma_A \colon \mathbb R \to G$
is the unique one-parameter subgroup of $G$ whose tangent vector at the identity is $A$. 
The \textbf{infinitesimal action} of $\g$ on $X$, $d\rho_X: \g \rightarrow \text{Vect}(X)$, is defined as the differential of the map $\rho_X$, that is,
\begin{align}
\label{eq:fund_vectorfields}
    d\rho_X
    : A \mapsto \frac{d}{d\tau}\bigg|_{\tau=0} \rho_X (\exp(\tau A)) (\bx) \equiv \Pi_A(\bx)~.
\end{align}
$\Pi_A$ is called the \textbf{fundamental vector field} corresponding to $A\in \g$.
Given a fixed point $\bx_0\in X$, we denote $\tau = \xi_A(\bx_0)^{-1}(\bx)$ the \textbf{fundamental flow coordinate}, which is the parameter such that applying the flow to $\bx_0$ gives $\bx$.
Central to our discussion is the fact that any smooth vector field $V: X \rightarrow TX$ on $X$ can be interpreted as a differential operator acting on smooth functions $f:X \rightarrow \R$. The operator $V(f)$ represents the directional derivative of $f$ at $\bx\in X$ in the direction of $V(\bx)$. We denote $\cL_A=\Pi_A \cdot \nabla$ the differential operator corresponding to $\Pi_A$. In the following we will use both $\Pi_\tau$ and $\Pi_A$ interchangeably, when no potential confusion arises.
When $\dim\g>1$ we indicate as $\bPi(\bx) = \begin{pmatrix}
    \Pi_{A_1} & \Pi_{A_2} & \cdots
\end{pmatrix}$ the matrix of the collection of fundamental vector fields.

Let us work out the example for $ X = \mathbb{R}^2 $ and $ G = \SO(2) $, the group of rotations in the plane. 
The Lie algebra $\mathfrak{so}(2)$ 
consists of all matrices of the form
$
A_\alpha = \begin{pmatrix}
0 & -\alpha \\
\alpha & 0
\end{pmatrix},
$
where $ \alpha\in\R$, and the Lie bracket is identically zero.
The flow on $X$ induced by $\rho_X$ is given by the exponential map
{\small
$
\rho_{\R^2}(\exp(\tau A_\alpha))(\bx) = \begin{pmatrix}
\cos(\alpha \tau) & -\sin(\alpha \tau) \\
\sin(\alpha \tau) & \cos(\alpha \tau)
\end{pmatrix} \bx$}, 
and without loss of generality we can set $\alpha=1$. 
The infinitesimal action is computed as
\begin{align}
\nonumber
d\rho_{\R^2}(A) 
= \frac{d}{d\tau}\bigg|_{\tau=0} \begin{pmatrix}
\cos \tau & -\sin \tau \\
\sin\tau & \cos \tau
\end{pmatrix}
\begin{pmatrix}
x_1 \\
x_2
\end{pmatrix} = \begin{pmatrix}
- x_2 \\
x_1
\end{pmatrix}~.
\end{align}
and thus the fundamental vector field defines the derivation $\cL_{A}(\mathbf{x}) = - x_2 \frac{\partial}{\partial x_1} + x_1 \frac{\partial}{\partial x_2}$.
Let  $\bx_0 \in \R^2$
be a fixed point, then the flow equation $\bx(\tau) \equiv \xi_{A}(\bx_0, \tau) = \rho_{\R^2}(\exp(\tau A),\bx_0)$
 gives a system of two equations, which we can solve to find the expression of the fundamental flow coordinate
\begin{align}
\begin{cases}
    \bx \cdot \bx_0 &= |\bx_0|^2 \cos\tau~, 
    \\
    \bx \times \bx_0 &= |\bx_0|^2 \sin\tau~, 
\end{cases} \ \ \Rightarrow \ \tau = \arctan \frac{\bx \times \bx_0}{\bx \cdot \bx_0}~,
\quad \text{where} \ \bx \times \by = y_2x_1 - x_1y_2~.
\end{align}
Note that
$\frac{\p}{\p\tau} = \frac{\p x_1}{\p \tau} \frac{\p}{\p x_1} + \frac{\p x_2}{\p \tau} \frac{\p}{\p x_2}
    = - x_2 \frac{\p}{\p x_1} + x_1 \frac{\p}{\p x_2} = \Pi_{A}(\bx)^\top \nabla = \cL_{A}~.$

\subsection{Intuition behind Lie group-induced generalized score matching}

Score matching aims at estimating a (log) probability density $p(\bx)$ by learning to match its score function, i.e., its gradient in data space. Generalized score matching replaces the gradient operator with a general linear operator $\bcL$ \citep{lyu2009interpretation}. The learning objective is given by minimizing the generalized Fisher divergence
\begin{align}
\label{eq:generalized_fisher}
    D_{\bcL}(p||q_\theta) = \int_X p(\bx) \left|\bcL \log p(\bx) - \bs_\theta(\bx) \right|^2 d\bx~,
\end{align}
where $\bs_\theta = \bcL\log q_\theta$. The requirement on the choice of $\bcL$ is that it preserves 
all the information about the original density. Formally, we require $\bcL$ to be \textit{complete}, that is, given two densities $p(\mathbf{x})$ and $q(\mathbf{x})$, $\bcL p(\mathbf{x}) = \bcL q(\mathbf{x})$ (almost everywhere \footnote{Almost everywhere means everywhere except for a set of measure zero, where 
we assume the standard Lebesgue measure.}) implies that $p(\mathbf{x}) = q(\mathbf{x})$ (almost everywhere). 

Given a Lie group $G$ acting on $X$, the collection of fundamental fields $\mathbf{\Pi}$ corresponding to a choice of basis $\bA=(A_1, A_2, \dots)$ of $\g$ is a linear operator, thus potentially suitable for score-matching.
It is then natural to set $\bcL$ to the derivation associated with the fundamental fields $\bPi$, i.e., $\bcL = \bPi(\bx)^\top \nabla$. 
It then follows that $\bcL \log p(\bx)$ computes the directional derivatives
of $\log p(\bx)$ with respect to the fundamental flow coordinates $\boldsymbol{\tau}$,
and provided that $\mathbf{\Pi}$ meets some consistency conditions (which we will address in the next section), we can employ $\bcL \log p(\mathbf{x})$ to sample from $p(\mathbf{x})$ using Langevin dynamics:
\begin{align}
\label{eq:langevin_schematic}
    \bx_{t+1} &= 
    \bx_{t} - \bs_\theta(\bx_t) d\rho_X(\exp(\btau \bA) )(\bx_t)
    = \bx_t - \sum_i\underbrace{\cL_i \log p_t(\bx_t)}_{\text{generalized scores}}  \underbrace{ \Pi_{A_i}(\bx_t)}_{\text{$A_i$ directions}}\Delta t~,
\end{align}
where $\Delta t $ is the step size and we have temporarily set aside stochasticity and denoising aspects.
This process mirrors the intuition depicted in Figure \ref{fig:trajectories}: each infinitesimal step of the dynamics corresponds to infinitesimal transformations along the flow on $X$ induces by the $G$-action, and each component of the generalized score is learned through maximum likelihood over the orbits $\xi_{A_i}$ of the corresponding transformations.

\subsection{Sufficient conditions for Lie group-induced generalized score matching}
\label{ss:g_conditions}

We now address the properties our setup ($X$, $G$, $\g$, $\bPi$) must satisfy to meet the sufficient conditions for score-matching and Langevin dynamics. We note that these results hold for any differentiable manifold $X$. Proofs for these results can be found in Appendix \ref{app:proofs_conditions}.
\paragraph{Condition 1: Completeness of $\bPi$.} We start by establishing an algebraic-geometric condition for $\bPi$'s completeness: 
\begin{proposition}
\label{prop:Pi_complete}
    The linear operator $\bPi(\bx)$ is complete if $\bPi$ is the local frame of a vector bundle $E$ over $X$ whose rank is $n\geq \dim X$ almost everywhere. If $\rank\ E = n$ everywhere, then $E=TX$, the tangent bundle of $X$. 
\end{proposition}
The following result specifies which Lie groups yield operators $\bPi$ satisfying the above proposition:
\begin{proposition}
\label{prop:main}
    The operator $\bPi$ induced by $\g$ is complete if and only if 
    the subspace $U\subseteq X$ such that $\dim \frac{G}{G_\bx} < n$ for $\bx\in U$, where $n=\dim X$, has measure zero in $X$.
\end{proposition}
As an example, consider standard score-matching on mass-centered point clouds. Here $X = \mathbb{R}^{3N-3}$, since the points' coordinates satisfy $\sum_{i=1}^N \mathbf{x}_i = 0$. Without loss of generality, $X$ can be parametrized by $\mathbf{x}_{1,\dots,{N-1}}$, with $\mathbf{x}_N$ determined by the center of mass condition. The group $G = T(3N)$ acts transitively on $X$, with a 3-dimensional stabilizer subgroup $G_X = \{(0, \dots, 0, \mathbf{a})^\top \in \mathbb{R}^{3N} \}$ fixing the space. Thus, $\dim G/G_X = n$ for all $\mathbf{x} \in X$, 
satisfying Proposition \ref{prop:main}.

\paragraph{Condition 2: Homogeneity of $X$.}While the completeness of the operators is necessary for estimating the target density, it is not sufficient to ensure that the Langevin dynamics \eqref{eq:langevin_schematic} will behave appropriately, as the following example illustrates. 
Let $X=\R$, and $G=\R^\ast_+$, the multiplicative group of non-zero positive real numbers.
The orbits under the action $\rho_X(a, x) = ax$ are $\cO_+ = (0, \infty)$, $\cO_-=(-\infty, 0)$, and $\cO_0= \{0\}$.
If the dynamics begins within $\cO_+$, it will be never be able to reach values in $\cO_-$, as $G$-transformations cannot move the system outside its initial orbit. We therefore ask that each pair of points of $X$ is connected through the $G$ action. This amounts to require that $X$ is \textbf{homogeneous} for $G$, that is, $\forall \bx, \by \in X$ there exists a $g \in G$ such that $\rho_X(g)\bx = \by$.
We note that this condition solely ensures that the generation outcome is independent of the initial sampling condition, that is, that Langevin dynamics can generate any point of the target distribution from any point of the prior. Beyond this, the formalism remains fully applicable in the non-homogeneous case, where the dynamic is restricted to orbits of the group, effectively partitioning the distribution. Though the formalism still applies within each orbit, global generation across $X$ would not be supported without homogeneity.

\paragraph{Condition 3: Commutativity of $\bPi$.}The final requirement is that $\bPi$ forms a (locally) \textbf{commuting} frame of vector fields, $[\cL_A, \cL_B]f(\bx)=0$ $\forall A,B$ and $\forall f\in C^{\infty}(X)$.
In this case, the coordinates $\tau_i$'s are orthogonal, and their flows commute, meaning the orbits parametrized by $\tau_i$ correspond to $ \{ \tau_j = 0 \}_{j \neq i} $. For non-commuting flows this is not the case, as Figure \ref{fig:fundamental}b-c illustrates: (b) $V_1 = x_1\partial_{x_1} + x_2\partial_{x_2},\ V_2 = x_1\partial_{x_2} - x_2\partial_{x_1}$ satisfy $[V_1, V_2] = 0$, and the orbits parametrized by $\tau_1=r$ correspond to subspaces with constant $\tau_2 = \theta$; (c) $W_{1,2} = V_{1,2} / |\mathbf{x}|$ do not commute, and the loci $\theta = \text{const}$ no longer coincide with the $r$-orbits, causing $\theta$ to vary along these, despite the fact that $r,\theta$ are still orthogonal at each point.
This last condition ensures that the updates governed by the different elements $A_i$ of $\g$ in \eqref{eq:langevin_schematic} remain independent of one another. 
Notably, this does not exclude non-Abelian groups; even if $A_{1,2}\in\g$ do not commute in the Lie algebra, their flows on $X$ can, as shown in the $\g=\so(3)$ example in Appendix \ref{appendix:spherical-so(3)-dilation}.

\section{Lie algebra score-based generative modeling via SDEs}

In this section, we formalize the framework we developed above
from the point of view of SDEs. Namely, we show that there exists a class of SDEs, which, when reversed, can generate data according to dynamics similar to \eqref{eq:langevin_schematic}, guided by the generalized score of the fundamental vector fields of the Lie algebra $\g$. 
Our main result is the following.
\begin{theorem}
\label{th:main_theorem}
    Let $G$ be a Lie group acting on $X$ satisfying the conditions of Section \ref{ss:g_conditions}, and let $\g$ be its Lie algebra. 
    The pair of SDEs
    {\small
    \begin{align}
    \label{eq:SDEforward}
     d\bx & =
        \left[\beta(t)\bPi(\bx) \bbf(\bx) +\frac{\gamma(t)^2}{2} \rho_X(\Omega) \right]dt  + \gamma(t) \bPi(\bx) d\bW~,\\
        \label{eq:SDEbackward}
              d\bx &= \left[ \beta(t)\bPi(\bx)\bbf(\bx)  - \frac{\gamma^2(t)}2 \rho_X(\Omega)
     -\gamma^2(t) \bPi(\bx) \nabla^\top \cdot \bPi(\bx)
     \right. \nonumber\\
     &\qquad\qquad\qquad\qquad\qquad\qquad \left.
     - \gamma(t)^2\bPi(\bx)\bcL \log p_t(\bx)\right] 
           dt +\gamma(t) \bPi(\bx)  d\bW ~,
\end{align}}
where $\beta,\gamma:\R \rightarrow \R$ are time-dependent functions, $\bPi: \R^n \rightarrow \R^{n\times n}$ the fundamental vector fields,  $\bbf:\R^n \rightarrow \R^n$ the drift, $\Omega = \sum_i A_i^2$ is known as the quadratic Casimir element of $\g$, and $\bcL=\bPi(\bx)^\top \nabla $, is such that
\begin{enumerate}
    \item The forward-time SDE \eqref{eq:SDEforward} is exactly solvable: 
    {\small
    \begin{align}
    \label{eq:SDE_solution}
        \bx(t) = \left(\prod_i O_i(\tau_i(t)) \right)\bx(0) = \left( \prod_{i=1}^ne^{\tau_i(t) A_i} \right) \bx(0)~, 
    \end{align}
    }
    where $O_i=e^{\tau_i(t) A_i}$ is the finite group action and $\btau(t)$ is the solution to the SDE
    \begin{align}
    \label{eq:forward-sde-flow-coordinates}
        d\btau(\bx) 
        & = \beta(t)\bbf(\bx)dt +\gamma(t) d\bW ~.
    \end{align}
    \item The SDE \eqref{eq:SDEbackward} is the reverse-time process of \eqref{eq:SDEforward}.
    \item The Langevin dynamic of the above SDEs decomposes as a direct sum of $\g$ infinitesimal actions \eqref{eq:fund_vectorfields}, defining an infinitesimal transformation along the flows $\xi_{\btau}$.
\end{enumerate}
\end{theorem}
We refer to Appendix \ref{app:proof_theorem} for the full proof of the above result. Here we limit
ourselves to a few comments regarding the extra terms that appear in the paired SDEs.
\begin{wrapfigure}{r}{0.2\columnwidth}
  \centering
  \includegraphics[width=0.2\columnwidth]{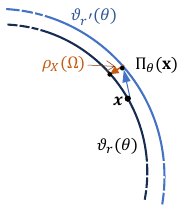}
  \caption{Quadratic Casimir for $\SO(2)$.} 
  \label{fig:casimir}
  \vspace{-0.5cm}
\end{wrapfigure}
The appearance of the Casimir element (we assume the identity as bilinear form on $\g$ \citep{kac1983invariant}) compensates for the deviation of the tangent vector from the orbit due to the curvature of the flow coordinates.
This can be seen in the example of $\SO(2)$ acting on $\R^2$  (which will be discussed thoroughly below).
 An infinitesimal transformation along the $\theta$ direction, represented by $\Pi_\theta$, moves any point $\mathbf{x}$ along a vector tangent to its $\SO(2)$ orbit, a circle of radius $r = \sqrt{x_1^2 + x_2^2}$. Due to the orbit's non-zero curvature, this movement would shift the point to an orbit of radius $r'>r$. The term $\rho_X(\Omega)$ 
compensates for this displacement, ensuring the final point remains close to the original orbit. This is illustrated in Figure \ref{fig:casimir}.
With this result at hand we can formulate our procedure for our Lie group-induced score-based generative modeling with SDEs. 

\paragraph{Perturbing data through the SDE.}

The forward-time SDE \eqref{eq:SDEforward} defines a noising
diffusion process respecting the decomposition of 
the Lie algebra $\g$ infinitesimal actions on $X$. 
In fact, given a data sample $\bx(0)\sim p_0$, the solution \eqref{eq:SDE_solution} takes the form of a product of finite group element actions $O_i$ on $\bx(0)$, 
where the specific order is irrelevant since the Lie algebra generators commute. 
For each factor, we first determine $\btau(0)=\btau(\bx(0))$, and employ these
as initial conditions for the forward SDE \eqref{eq:forward-sde-flow-coordinates}. By choosing appropriately the drift terms $f_i's$, for instance, to be affine in the flow coordinates $\tau_i$,
we can solve for $\btau(t)$ with standard techniques \citep{sarkka2019applied}, as this will follow a Gaussian distribution.
Alternatively, we can sample from  $\btau(t)$ by first simulating \eqref{eq:forward-sde-flow-coordinates}, then performing sliced score matching \citet{song2020sliced, pang2020efficient}
to sample from $p_t(\bx(\btau(t))| \bx(0))$. 

\paragraph{Generating samples through the reverse SDE.}

The time-reverse SDE \eqref{eq:SDEbackward} guides the generation of samples  $\bx(0)\sim p_0(\bx)$ starting from samples $\bx(T)\sim p_T(\bx)$, provided we can estimate the generalized score $\bcL \log p_t(\bx)$
of each marginal distribution. 
To sample from $p_T$, we use the fact that the distribution in the flow coordinates $\btau$ is tractable (with an appropriate choice of the drift terms $\bf$ and time-dependent functions $\beta,\gamma$ in \eqref{eq:SDEforward}), and that
(since $p_t(\bx)d\bx = p_t(\btau)d\btau$)
\begin{align}
\label{eq:distribution_time}
    p_t(\bx) = p_t(\btau) \left|\frac{\p \btau}{\p\bx}\right| 
    =  p_t(\btau) \left|\bPi^{-1}(\bx)\right|  ~,
\end{align}
where the extra term corresponds to the determinant of the Jacobian of the coordinate transformation induced by the fundamental flow coordinates. 
In particular, when $\bbf(\btau)$ is affine in $\btau$, it follows that $p_T(\btau) = \mathcal{N}(\boldsymbol{\tau} \mid \boldsymbol{0}, \bSigma)$,
where $\bSigma = \text{diag}(\sigma_1^2, \sigma_2^2, \dots, \sigma_n^2)$. Thus, we can sample $\btau(T)\sim p_T(\btau)$
simply as a collection of independent Gaussian random variables, and use the flow map  
to obtain $\bx(T)=\boldsymbol{\xi}_{\bA}(\btau(T),\bx_0)$, which will follow the distribution \eqref{eq:distribution_time} for $t=T$.
We describe training and sampling procedures in Algorithms \ref{alg:training} and \ref{alg:sampling} with further details for different groups in Appendix \ref{app:experiments}.
\algrenewcommand\algorithmicindent{0.5em}%
\begin{figure}[t]
\begin{minipage}[t]{0.495\textwidth}
\begin{algorithm}[H]
  \caption{Training with variance-preserving scheduler} \label{alg:training}
  \small
  \begin{algorithmic}[1]
    \Repeat
      \State $\bx_0 \sim q(\bx_0)$
      \State $t \sim \mathrm{Uniform}(\{1, \dotsc, T\})$
      \State $\boldsymbol{\eta}\sim\mathcal{N}(\bzero,\bI)$
      \State $\btau_0 = M_G(\bx_0)$ \Comment{Flow coordinates. 
      $M_G$ is group-dependent}
      \State $\btau_t = \alpha_t \btau_0 + \sigma_t\boldsymbol{\eta}$ \Comment{Sample from $p(\boldsymbol{\tau}_t|\boldsymbol{\tau}(\bx_0))$}
      \State $\bx_t = M_G^{-1}(\btau_t) $ \Comment{Cartesian coordinates} 
      \State Take gradient descent step on
      \Statex $\qquad \grad_\theta \left\| \mathbf{s}_\theta(\bx_t, t) + \frac{\boldsymbol{\eta}}{\sigma_t} \right\|^2$
    \Until{converged}
  \end{algorithmic}
\end{algorithm}
\end{minipage}
\hfill
\begin{minipage}[t]{0.495\textwidth}
\begin{algorithm}[H]
  \caption{Sampling with variance-preserving scheduler} \label{alg:sampling}
  \small
  \begin{algorithmic}[1]
    \vspace{.04in}
    \State $\btau_T \sim \mathcal{N}(\bzero, \bI)$
    \State $\bx_T = M_G^{-1}(\btau_T)$ \Comment{Cartesian coordinates}
    \For{$t=T, \dotsc, 1$}
      \State $\boldsymbol{\eta} \sim \mathcal{N}(\bzero, \bI)$ if $t > 1$, else $\boldsymbol{\eta} = \bzero$
      \State $ \mathbf{v}_{s,t} = (\sum_i \underbrace{(\frac{1}{2}\btau_{t,i} + \mathbf{s}_\theta(\bx_t, t)_i)}_{\in \R} {A}_i )\bx_t$ 
      \Comment{Dynamics induced by drift and generalized scores}
      \State $\rho_X(\Omega)= \mathbf{v}_{c, t} = (\sum_{i} {A}_{i}^2)\bx_t $ \Comment{Dynamics induced by quadratic Casimir elements}
      \State $\mathbf{v}_{d,t} = \sum_i \underbrace{\nabla \cdot [{A}_i \bx_t]}_{\in \R}({A}_i\bx_t) $ \Comment{Dynamics induced by divergences}
      \State $\mathbf{v}_t = \mathbf{v}_{s,t} + \frac{1}{2}  \mathbf{v}_{c,t} +  \mathbf{v}_{d,t} $
      \State $\tilde{\bx}_{t-1} = \bx_t + \beta_t \mathbf{v}_t$  \Comment{Update state based on velocity}
      \State $\bx_{t-1} = \tilde{\bx}_{t-1} + \sqrt{\beta_t}\sum_i{\eta_i {A}_i\bx_t}$ \Comment{Stochastic dynamics}
      \State $\btau_{t-1} = M_G(\bx_{t-1})$
    \EndFor
    \State \textbf{return} $\bx_0$
    \vspace{.04in}
  \end{algorithmic}
\end{algorithm}
\end{minipage}
\vspace{-1em}
\end{figure}

\paragraph{Estimating the generalized score.}

Analogously to standard score matching, we train a time-dependent neural network $\bs_{\btheta}(\bx(t), t):\R^n \times \R \rightarrow \R^n$ to estimate the generalized score $\bcL \log p_t(\bx(t)| \bx(0))$ at any time point, that is,
we minimize the objective
\begin{align}
\label{eq:score_objective}
    &\E_t \left\{w(t) \E_{\bx(0)\sim p_0(\bx)} \E_{\bx(t)\sim p_t(\bx|\bx(0))} \left[ \big| \bs_{\btheta}(\bx(t), t) - \bcL \log p_t(\bx(t)| \bx(0)) \big|^2 \right] \right\}~,
\end{align}
where $w:[0, T]\rightarrow \R_+$ is a time-weighting function. 
Now, from Condition 3 above and the property that $\cL_{A_i}$ computes the direction derivative along the flow of $\Pi_{A_i}(\bx)$, it follows that $\bcL \log p_t(\bx(t) | \bx(0)) = \nabla_{\btau(t)}\log p_t(\bx(\btau)(t) | \bx(\btau)(0))$.
Under the above assumptions, $p_t(\btau) = \mathcal{N}(\boldsymbol{\tau} | \bmu(\bx(0), t), \bSigma(t))$, where the  form of the mean and the variance 
depends on the explicit form of \eqref{eq:forward-sde-flow-coordinates}. Using the parametrization $\btau(t) = \bmu(\bx(0), t) + \sqrt{\bSigma(t)}\boldsymbol{\eta}_t$, where $\boldsymbol{\eta}_t\sim \cN(\boldsymbol{0},\boldsymbol{I})$, we obtain
\begin{align}
    \bcL \log p_t(\bx(t) | \bx(0)) &= -\bSigma^{-1}(\btau(t)-\bmu(\bx(0), t) = -\sqrt{\bSigma(t)}^{-1} \boldsymbol{\eta}_t~.
\end{align}

\subsection{Examples}\label{subsection:examples}

In this section we look at some relevant examples for different choices for $G$ and $X$.

\paragraph{Standard Score Matching.}\label{sec:standard-score-matching}

Standard score matching can be recovered as a special case of our formalism by choosing $X=\R^n$ and $G=T(n)$. As we show explicitly in Appendix \ref{app:TN}, we have $\bcL=\nabla$
and the Lie algebra action $\bPi(\bx) = \boldsymbol{I}$, the identity on $X$. Since $\bPi$ is $\bx$-independent, its divergence vanishes, as well as the quadratic Casimir ($T(N)$ is Abelian), so that the SDEs \eqref{eq:SDEforward} take the known form
\begin{align}
    d\bx & = \beta(t) \bbf(\bx) dt+\gamma(t) d\bW~, 
    &d\bx &=\left[ \beta(t)\bbf(\bx) - \gamma(t)^2\nabla \log p_t(\bx)\right] 
           dt +\gamma(t)  d\bW~. 
\end{align}

\paragraph{\texorpdfstring{$\boldsymbol{G=\SO(2)\times \R_+}$.}{G=SO(2)X R+}}

A simple but non-trivial case in given by $G=\SO(2)\times \R_+$ describing rotations and dilations acting on $X=\R^2$. 
A basis for 
$\g = \so(2)\oplus\R$ is given by $A_r = \boldsymbol{I}$ and $A_\theta=\begin{pmatrix}
    0 & -1\\1 & 0
\end{pmatrix}$, yielding $\bPi(\bx)=\begin{pmatrix}
    x & -y\\ y & x
\end{pmatrix}$, which satisfies all the conditions of section \ref{ss:g_conditions}. Following our discussion above and in Appendix \ref{app:TN} we have 
(since $\rho(\Omega) = A_r^2 + A_\theta^2 = \boldsymbol{I} - \boldsymbol{I} = \boldsymbol{0}$)
\begin{align}\label{eq:sde-2d-ob}
     d\bx & = \beta(t)\left(f_r(r)A_r \bx + f_\theta(\theta) A_\theta \bx \right)dt 
 +\gamma(t)\left(dW_rA_r \bx + dW_\theta A_\theta \bx \right)  ~,
\end{align}
and we see that the SDE splits into contributions from the two Lie algebra summands.
To find an explicit solution, let $\gamma(t)=\sqrt{\beta(t)}$ and $f_r = -\frac14 \log(x^2+y^2)$, $f_\theta = -\frac{1}{2}\arctan\frac{y}{x}$. This choice corresponds, in the flow coordinates system, to a 2d Ornstein-Uhlenbeck system \citep{gardiner1985handbook}
which has a Gaussian solution with mean $\begin{pmatrix}
        r(0) \\
        \theta(0)
    \end{pmatrix} e^{-\int_0^t \beta(s)ds}$ and variance $\left( 1- e^{-\int_0^t \beta(s)ds}\right) I$.
Let us define $\sigma(t) = \sqrt{1- e^{-\int_0^t \beta(s)ds}}$, such that 
$r(t)= r(0) + \lambda(t) = r(0) - r(0) \sigma(t)^2 + \sigma(t) \eta_r$ and 
similarly $\theta(t)= \theta(0)+ \varphi(t) = \theta(0) - \theta(0) \sigma(t)^2 + \sigma(t) \eta_\theta,$
where $\eta_r, \eta_\theta\in \cN(0,1)$, then it is an easy calculation to show that
{\small
\begin{align}
\label{eq:back-reverse-cartesian}
    \begin{pmatrix}
        x_1(t) \\ x_2(t)
    \end{pmatrix} 
    &=
    e^{\lambda(t)}
    \begin{pmatrix}
        \cos\varphi(t) & -\sin\varphi(t) \\ \sin\varphi(t) & \cos\varphi(t)
    \end{pmatrix}
    \begin{pmatrix}
        x_1(0) \\ x_2(0)
    \end{pmatrix}~.
\end{align}
}
We can look at the asymptotic behavior of the solution. 
Assuming that $\beta(t)$ is a monotonous increasing function, that is, $\beta(t + \epsilon) > \beta(t)$ for $\epsilon > 0$, then $\lim_{t\rightarrow\infty} \sigma(t) = 1$ 
and hence
{\small
\begin{align}
\label{eq:2dorsteinuhelnbeck}
    \lim_{t\rightarrow \infty} \bx(t)
    &= e^{-r_0+\eta_r}
    \begin{pmatrix}
        \cos\theta_0 & \sin\theta_0 \\ -\sin\theta_0 & \cos\theta_0
    \end{pmatrix}
     \begin{pmatrix}
        \cos\eta_\theta & \sin\eta_\theta \\ -\sin\eta_\theta & \cos\eta_\theta
    \end{pmatrix} 
    \begin{pmatrix}
        e^{r_0}\cos\theta_0 \\ e^{r_0}\sin\theta_0
    \end{pmatrix}
    =\begin{pmatrix}
        e^{\eta_r} \cos\eta_\theta \\ e^{\eta_r}\sin\eta_\theta
    \end{pmatrix}~,\nonumber
\end{align}
}
where $\theta_0=\theta(0), \ r_0=r(0)$. Note that, even if \eqref{eq:2dorsteinuhelnbeck} is not Gaussian, we can still
easily draw samples from it by sampling the two Gaussian variables $\eta_{r,\theta}$.

\paragraph{Dihedral and bond angles.}
\label{ss:torsion}
\begin{wrapfigure}{r}{0.5\columnwidth}
    \centering
    \includegraphics[width=0.5\columnwidth]{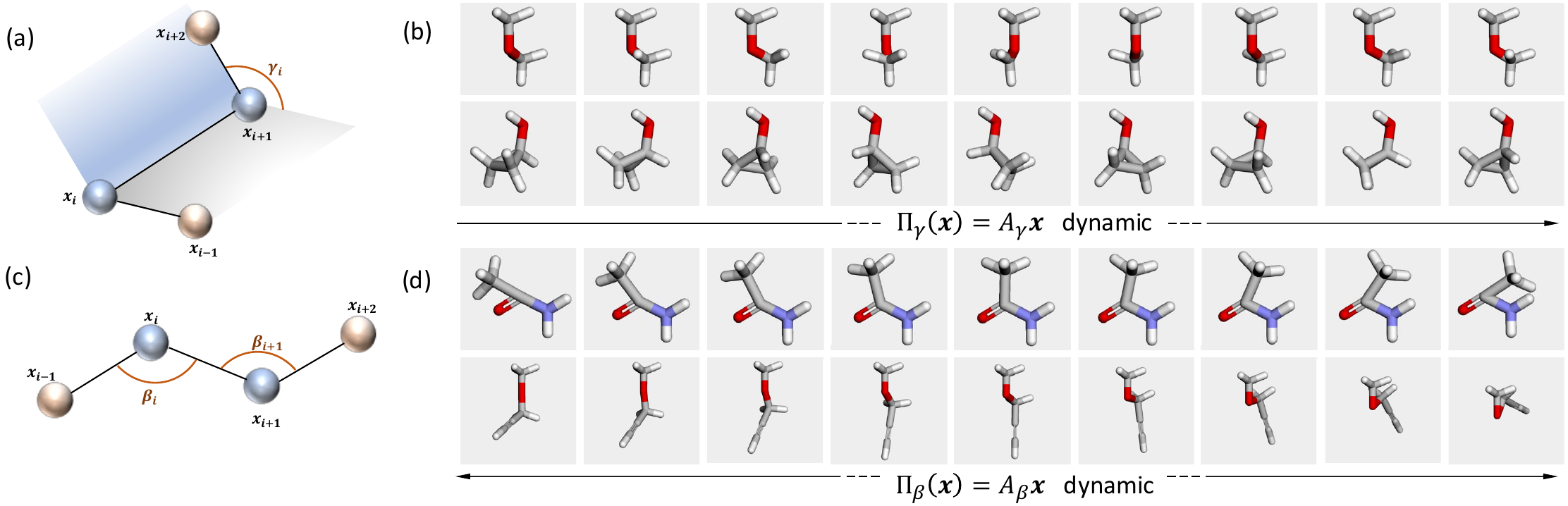}
    \caption{Lie algebra $\so(2)\subset\so(3)$ dynamics for torsion (a,b) and bond angles (c,d) in molecular conformers.}
    \label{f:torsion}
\end{wrapfigure}

The above formalism can be applied to obtain transformations of physically meaningful quantities, as bond and torsion angles for molecules' conformations. Let $\gamma_i$ be the dihedral angle
between the planes identified by the points $\{\bx_{i-1}, \bx_i, \bx_{i+1}\}$ and $\{\bx_{i}, \bx_{i+1}, \bx_{i+2}\}$, respectively (Figure \ref{f:torsion}a). The Lie algebra element corresponding to an infinitesimal change 
in $\gamma_i$ is given by a $3N\times3N$-dimensional $3\times3$-block diagonal matrix, whose 
$j=1,\dots,N$ block is given by $H(j-(i+1))\widehat{\bx}_{i+1,i}\cdot \bA)$, where $\bA = (A_x, A_y, A_z)$ is the vector of the Lie algebra basis for $\so(3)$, $\hbx_{i+1,i}=(\bx_{i+1}-\bx_i)/|(\bx_{i+1}-\bx_i)|$ and $H(i)=1$ if $i>0$ and 0 otherwise is the Heaviside step function. For bond angles $\beta_i$ (Figure \ref{f:torsion}c) we
construct the corresponding $\so(2)\in \so(3)$ algebra element blocks as 
$H(j-i) (\bx_{i+1,i}\times \bx_{i-1,i})\cdot \bA$. 
Examples of the dynamics generated by these operators are presented in Figure \ref{f:torsion}(b,d).

\section{Related Work}

Representation theory applied to neural networks has been studied both theoretically \citep{esteves2020theoretical, chughtai2023toy, puny2021frame, smidt2021euclidean} and applied to a variety of groups, architectures and data type: CNNs \citep{cohen2016group, romero2020attentive, liao2023lie, finzi2020generalizing,  weiler2019general, weiler20183d}, Graph Neural Networks \citep{satorras2021n}, Transformers,  \citep{geiger2022e3nn, romero2020group, hutchinson2021lietransformer}, point clouds \citep{thomas2018tensor}, 
chemistry \citep{schutt2021equivariant, le2022equivariant}.
On the topic of disentanglement of group action and symmetry learning, \citet{pfau2020disentangling} factorize a Lie group from the orbits in data space,
while \citet{winter2022unsupervised} learn through an autoencoder architecture invariant and equivariant representations of any group acting on the data. 
\citet{fumero2021learning} learns disentangled representations solely from data pairs. 
\citet{dehmamy2021automatic} propose an architecture based on Lie algebras that can automatically discover symmetries from data.
\citet{xu2022geodiff} predict molecular conformations from molecular graphs
in an roto-translation invariant fashion with equivariant Markov kernels.

Related to our study is the field of diffusion on Riemannian manifolds. \citet{de2022riemannian} propose diffusion in a product space, a condition which is not a necessary in our framework, defined by the flow coordinates in the respective Riemannian sub-manifolds. When the Riemannian manifold is a Lie group, 
their method yields dynamics similar to ours, as illustrated in an example in Section \ref{ss:torsion}. 
In fact, our formalism could be integrated with their approach to create a unified framework for diffusion processes on the broader class of Riemannian manifolds admitting a Lie group action. 
These techniques has been applied in a variety of use cases \citep{corso2022diffdock, ketata2023diffdock, pmlr-v202-yim23a, jing2022torsional} for protein docking, ligand and protein generation.
The works \citet{zhu2024trivialized, kong2024convergence}
leverage trivialized momentum to perform diffusion on the Lie algebra (isomorphic to \(\mathbb{R}^n\)) instead of the Lie group, thereby eliminating curvature terms, although their approach is to date only feasible for Abelian groups. 
An interesting connection with our work is the work of \citet{kim2022maximum}: the authors propose a bijection to map a non-linear problem to a linear one, to approximate a bridge between two non-trivial distributions. Our case can be seen as a bijection between the (curved) Lie group manifold and the (flat) Euclidean data space.

In the context of interpreting the latent space \citep{bertolini2023explaining} of diffusion models, \citet{park2023understanding} explores the local structure of the latent space (trajectory) of diffusion models using Riemannian geometry. 
Similarly, \citet{haas2024discovering} propose a method to uncover semantically meaningful directions in the semantic latent space ($h$-space) \citep{wang2023infodiffusion} of denoising diffusion models (DDMs) by PCA. \citet{wang2023infodiffusion} propose a method to  learn disentangled and interpretable latent representations of diffusion models in an unsupervised way. 
We note that the aforementioned works aim to extract meaningful latent factors in traditional DDMs, often restricting to human-interpretable semantic features and focusing on image generation. 

\section{Experiments}
\label{s:experiments}

\subsection{2d, 3d and 4d distributions}
\label{ss:toys}

\begin{figure*}[t!]
    \centering
    \includegraphics[width=1.0\textwidth]{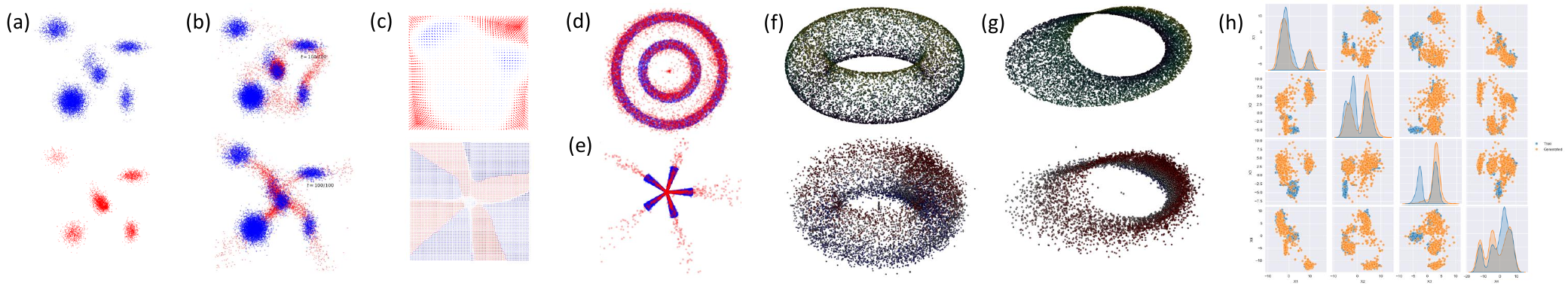}
    \caption{(a) 2d mixture of Gaussians (top: ground truth, bottom: generated); (b) generating process using single scores for the subgroups $\SO(2), \R_+$ with the corresponding 
    scores (c); (d,e) one-dimensional learning for a symmetric distributions; $3d$-distributions: torus (f) and Möbius strip (g) (top: ground truth, bottom: generated); (h) 4d mixture of Gaussian for the group $G=\SO(4)\times \R_+$.}
    \label{fig:2d_distributions}
\end{figure*}

In Figure \ref{fig:2d_distributions} we illustrate the framework for a variety of $d=2,3$-dimensional distributions. In all cases we take $G=\SO(d)\times \R_+$. Figure \ref{fig:2d_distributions}(a,b,c) displays a mixture of Gaussians:
in (a) (bottom) we see that our generalized score-matching can learn any distribution, regardless of its inherent symmetry;  (b) shows the output of the generation process using only one score (top $\g=\so(2)$, bottom $\g=\mathfrak{r}_+$), while (c) shows the vector fields corresponding to the scores, where we color-coded the field directions.
Figures \ref{fig:2d_distributions}(d,e) depicts radial and angular distributions, where the score is learned using the respective Lie algebra elements. This reflects the ability to leverage the symmetry properties of the data and perform diffusion in a lower-dimensional space. We also show in Figure \ref{fig:2d_distributions}h ($G=\SO)(4)\times\R_+$) that our method can be applied to higher dimensional Lie groups. We list quantitative comparisons in terms of W2-distances for our generalized score model against standard diffusion model in Appendix \ref{sec:appendix-toy}.
\vspace{-0.2cm}

\subsection{Rotated MNIST}
\label{ss:MNIST}

In this experiment we show that our framework can be applied to effectively learn a bridge between two non-trivial distributions, adopting however only techniques from score-matching and DDPM. 
Let $p_T(\bx)$ be the rotated MNIST dataset and $p_0(\bx)$ the original (non-rotated) MNIST dataset. 
We can learn to sample from $p_0$ starting from element of $p_T$ by simply modeling a $\SO(2)$ dynamic. 
Some examples of our results are shown in Figure \ref{fig:MNIST}. Notice that our formalism allows us to reduce the learning to a 1-dimensional score $\cL_\theta = x_1\p_{x_1} - x_2\p_{x_2}$, which reflects the true dimensionality of the problem. We trained the model with $T=100$ time-steps, but for sampling it suffices to set $T=10$. As it can be seen in the example trajectories \ref{fig:MNIST}b, the model starts converging already at $t/T\sim 0.5$. 
We employ a CNN which processes input images $\bx(t)$, and the resulting feature map is flattened and concatenated with a scalar input $t$, then passed through fully connected layers to produce the final output. 

We compare our approach to the Brownian Bridge Diffusion Model (BBDM) \citep{Li_2023_CVPR}. Unlike our method, BBDM operates unconstrained in the full MNIST pixel space (\(\mathbb{R}^{28 \times 28}\)), where intermediate states represent latent digits. As shown in Figure \ref{fig:MNIST}a, this can result in incorrect transitions, such as adding extraneous pixels or altering the original digit, even generating entirely different digits (Figure \ref{fig:MNIST}b).

\begin{wrapfigure}{r}{0.4\columnwidth}
\vspace{-0.4cm}
        \centering
        \captionof{table}{FID and Accuracy scores comparing GSM against BBDM.}
        \label{tab:gsm_bridge_summary}
        { \footnotesize
        \begin{tabular}{p{0.8cm}cc}
            \toprule
            Model & \begin{tabular}{c} Avg \\ Acc ($\uparrow$) \end{tabular} & \begin{tabular}{c} Avg \\ FID ($\downarrow$) \end{tabular} \\
            \midrule
            GSM & $\boldsymbol{0.96\pm 0.02}$ & $\boldsymbol{85.8\pm 15.7}$ \\
            BBDM & $0.80\pm 0.10$ & $133.4\pm 19.0$ \\
            \bottomrule
        \end{tabular}
}
\end{wrapfigure}
We further evaluate both methods on the classification accuracy as well as FID scores of generated MNIST digits. Since the task is to correctly rotate a MNIST digit into the correct orientation aligning with the ground-truth data distribution, we observe that our GSM model achieves superior classification accuracies (0.96 vs 0.80) and FID scores (85.77 vs 133.4) as shown in Table \ref{tab:gsm_bridge_summary}.
Further details can be found in Appendix \ref{sec:appendix-mnist}.

\begin{figure}[t]
    \centering
        \centering
        \includegraphics[width=\textwidth]{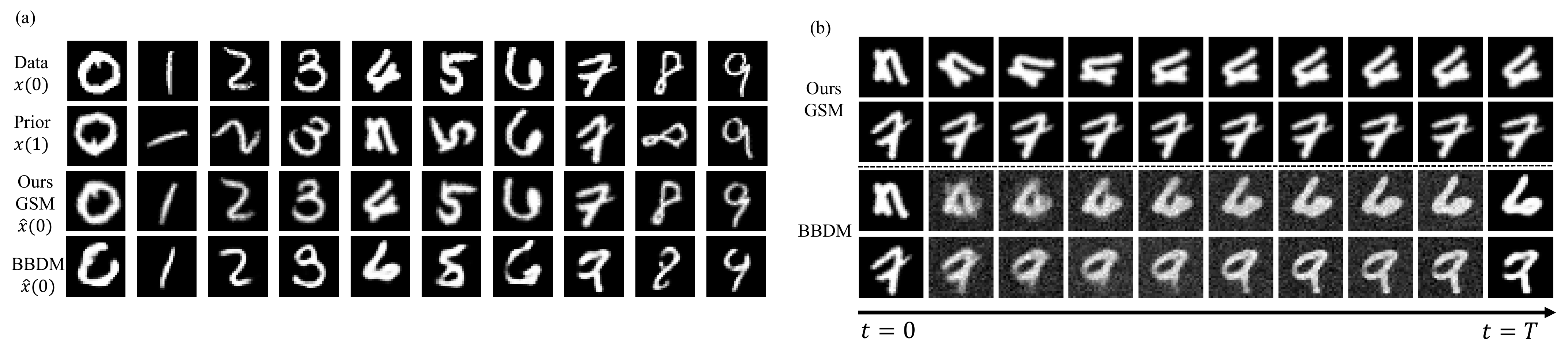}
        \caption{(a) Original and rotated MNIST samples with generated samples from our model and BBDM. (b) Reverse diffusion trajectories of our model against BBDM. Intermediate samples from BBDM resemble interpolation of mixed digits. For the first BBDM case, the 4-digit transitions into a 6-digit.}
        \label{fig:MNIST}
\end{figure}

\subsection{QM9.}
\label{ss:QM9}
We use our framework to train a generative model $p_{\theta}(X|M)$ for conformer sampling of small molecules $M$ from the QM9 dataset \citep{Ramakrishnan2014}. 
We only keep the lowest energy conformer as provided in the original dataset, that is, for each molecule only one 3D conformer is maintained. 
Here $X=\R^{3N}$ and we choose $G=(\SO(3)\times\R_+)^N$, where each factor acts on the space $\R^3$ spanned by the Cartesian coordinates of the molecule's atoms, respectively. 
\begin{wrapfigure}{r}{0.7\columnwidth}
    \centering
    \includegraphics[width=0.7\columnwidth]{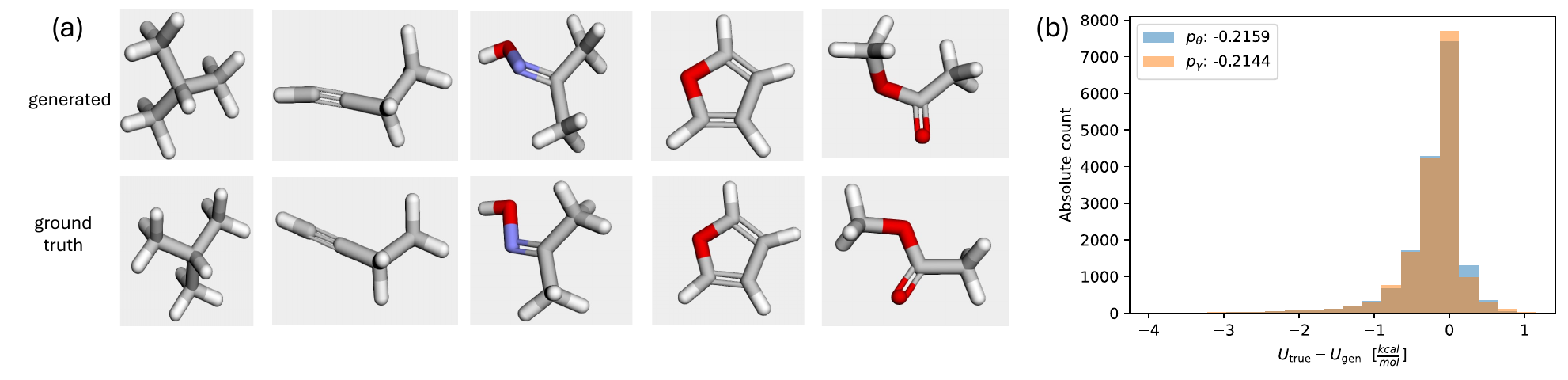}
    \caption{(a) Generated 3D conformer from the QM9 validation set (top) and ground truth conformer (bottom). (b) Energy difference distribution between diffusion models $(p_\theta, p_\gamma)$ and ground-truth energy. Both models generate a similar $\Delta$ energy distribution.}
    \label{fig:QM9}
    \vspace{-0.3cm}
\end{wrapfigure}
As Figure \ref{fig:QM9}a shows, our generative process yields conformers that are energetically very similar to the ground truth conformers, while showing some variability, as it can be seen in the last example where the torsion angle is differently optimized. We train another model $p_\gamma(X|M)$ via standard Fisher denoising score-matching, i.e., choosing $G=T(3)^N$ as in Sec. \ref{sec:standard-score-matching}, and generate 5 conformers per molecule for both models $p_\theta, p_\gamma$. We then compute the UFF energy \citep{uff_rappe_92} implemented in the RDKit for all generated conformers and extract the lowest energy geometry as generated sample. To compare against the reference geometry, we compute the energy difference $\Delta = U_\text{true}-U_\text{gen}$ for both models. Figure \ref{fig:QM9}b shows that both diffusion models tend to generate conformers that have lower energies than the ground true conformer according to the UFF parametrization, while the diffusion model that implements the dynamics according to $G=(\SO(3)\times\R_+)^N$ (colored in blue) achieves slightly lower energy conformers, mean $\Delta_\theta=-0.2159$ against 
mean $\Delta_\gamma=-0.2144$ for the standard diffusion model (colored in orange).

\subsection{CrossDocked2020: Global E(3) and Protein-Ligand Complexes.}
\label{ss:crossdocked}

In this final experiment, we train a generative model for global SE(3) transformations acting on small molecules. Specifically, given a pair consisting of a compound and a protein pocket, our goal is to generate the trajectory by which the ligand best fits into the pocket. Importantly, the internal structure of the compound remains fixed, which presents a challenge with standard diffusion processes. Thus, while the SE(3) transformations are global with respect to the ligand, they do not represent global symmetries of the overall system. We derive in appendix \ref{app:global_E3} the relevant operators that guide the dynamics \eqref{eq:SDEbackward}. 
Figure \ref{fig:cross_dock}a shows examples of docked molecules using $\SE(3)$-guided score-matching diffusion. The true and generated molecules at different generation steps are visualized as point clouds, showing a good agreement. Figure \ref{fig:cross_dock}b shows that our model achieves a lower RMSD ($2.9\pm 1.0$ \AA~vs $5.6\pm 1.2$ \AA) for the docked ligands than the RSGM method \citep{de2022riemannian, corso2022diffdock} (for details, we refer to Appendix \ref{app:rsgm}). We also compare our method against the Brownian Bridge Diffusion Model (BBDM) which operates on the $T(3)^N$ group, as a standard (Euclidean) diffusion baseline with the constraint to start and end at valid rigidly transformed molecules during training. We use the same network architecture as in the GSM and RSGM experiments to learn the correct \SO(3) rotation. Unlike existing experiments (our method and RSGM), the Euclidean BBDM in this setting attempts to learn only global 
 \SO(3) rotation, neglecting translation. Since the problem is implicitly 3-dimensional but the equivariant score network predicts all $3N$ ligand atom coordinates, final samples with implausible coordinate trajectories tend to have higher energies due to unphysical poses including bond stretching, non-planar aromatic rings, and deformed rings. 
In terms of mean/std RMSD on the CrossDocked2020 test set, our method (Lie algebra: $2.91 \pm 1.0 $ \AA ) is comparable with BBDM ($2.92 \pm 1.57$ \AA). However, since BBDM models all 
atomic coordinates, the overall dynamics do not follow a global \SO(3) rotation, achieving 
MAE$(D(x_0, x_0), D(\hat{x}_0, \hat{x}_0))= 0.43\pm 0.21$), while RSGM and our method achieve $0.0$ by design. This indicates that Lie algebra induced diffusion offers a clear advantage over standard diffusion models in this particular bridging problem.
\begin{figure}[t!]
    \centering
    \includegraphics[width=\linewidth]{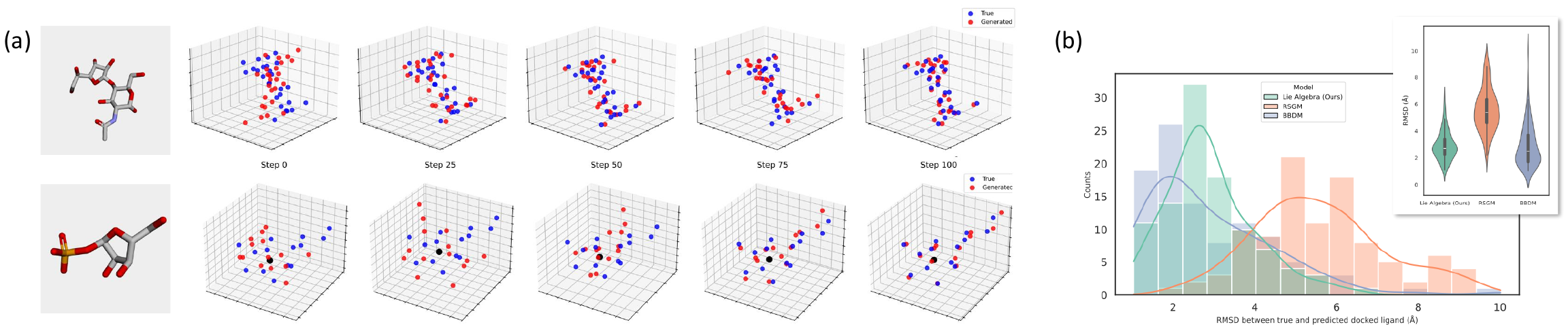}
    \caption{(a) $\SE(3)$ trajectories for molecular docking; (b) Comparison with RSGM.
    }
    \captionsetup{aboveskip=0.0pt, belowskip=2pt}
    \label{fig:cross_dock}
\end{figure}

\section{Conclusions and Outlook}

We presented a method for generative modeling on \textit{any} Lie group $G$ representation on a space $X$
through generalized score matching.
\textbf{Our framework generates a curved Lie group diffusion dynamics in flat Euclidean space}, 
without the need to transform the data and of performing group projections.
Specifically, we introduced a new class of exactly-solvable SDEs that guide the corruption and generation processes. 
Thus, our framework does not merely complement existing methods, but \textit{expands} the space of exactly solvable diffusion processes. 
Our framework is particularly relevant given recent findings \citep{abramson2024accurate} showing that unconstrained models outperform equivariant ones: with our framework there is no need of a tradeoff, as we retain the expressivity of unconstrained models on raw Cartesian coordinates with the benefits of group inductive bias. 
Moreover, our techniques descend quite straightforwardly to flow matching \citep{lipman2022flow} through the Diffusion Mixture Representation Theorem \citep{peluchetti2023non, brigo2008general}.
We spell out the connection in appendix \ref{app:flow_matching} and we plan to expand on this in future work.

In the context of generative chemistry, particularly for modeling interactions within protein pockets, our methods could be employed to decouple the intrinsic generation of ligands from the global transformations required to fit the ligand into the pocket. This approach can also be extended beyond 3D coordinates, for instance, by working with higher-order representations, such as modeling electron density \citep{Rackers_2023}.

Moreover, for more complex problems, it is feasible that an optimal generation process can be achieved by combining different choices of \(G\) along the trajectory. In the context of ligand generation, we propose a time-dependent group action $G_t = t T(3N) + (1-t)(\text{SO}(3) \times \mathbb{R}_+)^N$: at the beginning of the diffusion process, when the point cloud is still far from forming a recognizable conformer, we can leverage the properties of a true Gaussian prior. As the point cloud is gradually optimized to ``resemble a molecule'', we progressively transition to a generalized score-guided process. This shift allows us to fine-tune chemically relevant properties, such as bonds and torsion angles, ensuring that the intermediate and final conformers are chemically valid and accurate. This will be the focus of our forthcoming work. 

A potential limitation of our work is that it currently does not extend to representations of finite groups. While finite groups also admit a rich representation theory, it remains an open question how to adapt our framework to those settings. Another limitation lies in our assumption that $X$ is a vector space, whereas Lie groups can act on more general manifolds. Although the conditions discussed in Section~\ref{ss:g_conditions} hold for arbitrary manifolds, our main theoretical results are restricted to actions on vector spaces. Extending the full analysis to curved spaces would be a compelling direction for future work, potentially enabling a general theory of diffusion on Riemannian manifolds via Lie group representations.

\section*{Code Availability}
Our source code will be made available on \href{https://github.com/pfizer-opensource/symmetry-induced-score-matching}{https://github.com/pfizer-opensource/symmetry-induced-score-matching}.

\bibliography{references.bib}

\begin{thebibliography}{66}
\providecommand{\natexlab}[1]{#1}
\providecommand{\url}[1]{\texttt{#1}}
\expandafter\ifx\csname urlstyle\endcsname\relax
  \providecommand{\doi}[1]{doi: #1}\else
  \providecommand{\doi}{doi: \begingroup \urlstyle{rm}\Url}\fi

\bibitem[Abramson et~al.(2024)Abramson, Adler, Dunger, Evans, Green, Pritzel,
  Ronneberger, Willmore, Ballard, Bambrick, et~al.]{abramson2024accurate}
Josh Abramson, Jonas Adler, Jack Dunger, Richard Evans, Tim Green, Alexander
  Pritzel, Olaf Ronneberger, Lindsay Willmore, Andrew~J Ballard, Joshua
  Bambrick, et~al.
\newblock Accurate structure prediction of biomolecular interactions with
  alphafold 3.
\newblock \emph{Nature}, pp.\  1--3, 2024.

\bibitem[Anderson(1982)]{anderson1982reverse}
Brian~DO Anderson.
\newblock Reverse-time diffusion equation models.
\newblock \emph{Stochastic Processes and their Applications}, 12\penalty0
  (3):\penalty0 313--326, 1982.

\bibitem[Bertolini et~al.(2023)Bertolini, Clevert, and
  Montanari]{bertolini2023explaining}
Marco Bertolini, Djork-Arn{\'e} Clevert, and Floriane Montanari.
\newblock Explaining, evaluating and enhancing neural networks’ learned
  representations.
\newblock In \emph{International Conference on Artificial Neural Networks},
  pp.\  269--287. Springer, 2023.

\bibitem[Blumenson(1960)]{blumenson1960derivation}
LE~Blumenson.
\newblock A derivation of n-dimensional spherical coordinates.
\newblock \emph{The American Mathematical Monthly}, 67\penalty0 (1):\penalty0
  63--66, 1960.

\bibitem[Brehmer \& Cranmer(2020)Brehmer and Cranmer]{brehmer2020flows}
Johann Brehmer and Kyle Cranmer.
\newblock Flows for simultaneous manifold learning and density estimation.
\newblock \emph{Advances in neural information processing systems},
  33:\penalty0 442--453, 2020.

\bibitem[Brigo(2008)]{brigo2008general}
Damiano Brigo.
\newblock The general mixture-diffusion sde and its relationship with an
  uncertain-volatility option model with volatility-asset decorrelation.
\newblock \emph{arXiv preprint arXiv:0812.4052}, 2008.

\bibitem[Chughtai et~al.(2023)Chughtai, Chan, and Nanda]{chughtai2023toy}
Bilal Chughtai, Lawrence Chan, and Neel Nanda.
\newblock A toy model of universality: Reverse engineering how networks learn
  group operations.
\newblock In \emph{International Conference on Machine Learning}, pp.\
  6243--6267. PMLR, 2023.

\bibitem[Cohen \& Welling(2016)Cohen and Welling]{cohen2016group}
Taco Cohen and Max Welling.
\newblock Group equivariant convolutional networks.
\newblock In \emph{International conference on machine learning}, pp.\
  2990--2999. PMLR, 2016.

\bibitem[Corso et~al.(2023)Corso, St{\"a}rk, Jing, Barzilay, and
  Jaakkola]{corso2022diffdock}
Gabriele Corso, Hannes St{\"a}rk, Bowen Jing, Regina Barzilay, and Tommi~S.
  Jaakkola.
\newblock Diffdock: Diffusion steps, twists, and turns for molecular docking.
\newblock In \emph{The Eleventh International Conference on Learning
  Representations}, 2023.
\newblock URL \url{https://openreview.net/forum?id=kKF8_K-mBbS}.

\bibitem[De~Bortoli et~al.(2022)De~Bortoli, Mathieu, Hutchinson, Thornton, Teh,
  and Doucet]{de2022riemannian}
Valentin De~Bortoli, Emile Mathieu, Michael Hutchinson, James Thornton,
  Yee~Whye Teh, and Arnaud Doucet.
\newblock Riemannian score-based generative modelling.
\newblock \emph{Advances in Neural Information Processing Systems},
  35:\penalty0 2406--2422, 2022.

\bibitem[Dehmamy et~al.(2021)Dehmamy, Walters, Liu, Wang, and
  Yu]{dehmamy2021automatic}
Nima Dehmamy, Robin Walters, Yanchen Liu, Dashun Wang, and Rose Yu.
\newblock Automatic symmetry discovery with lie algebra convolutional network.
\newblock \emph{Advances in Neural Information Processing Systems},
  34:\penalty0 2503--2515, 2021.

\bibitem[Dhariwal \& Nichol(2021)Dhariwal and Nichol]{dhariwal_guidance_cosine}
Prafulla Dhariwal and Alexander~Quinn Nichol.
\newblock Diffusion models beat {GAN}s on image synthesis.
\newblock In A.~Beygelzimer, Y.~Dauphin, P.~Liang, and J.~Wortman Vaughan
  (eds.), \emph{Advances in Neural Information Processing Systems}, 2021.
\newblock URL \url{https://openreview.net/forum?id=AAWuCvzaVt}.

\bibitem[Esteves(2020)]{esteves2020theoretical}
Carlos Esteves.
\newblock Theoretical aspects of group equivariant neural networks.
\newblock \emph{arXiv preprint arXiv:2004.05154}, 2020.

\bibitem[Finzi et~al.(2020)Finzi, Stanton, Izmailov, and
  Wilson]{finzi2020generalizing}
Marc Finzi, Samuel Stanton, Pavel Izmailov, and Andrew~Gordon Wilson.
\newblock Generalizing convolutional neural networks for equivariance to lie
  groups on arbitrary continuous data.
\newblock In \emph{International Conference on Machine Learning}, pp.\
  3165--3176. PMLR, 2020.

\bibitem[Fumero et~al.(2021)Fumero, Cosmo, Melzi, and
  Rodol{\`a}]{fumero2021learning}
Marco Fumero, Luca Cosmo, Simone Melzi, and Emanuele Rodol{\`a}.
\newblock Learning disentangled representations via product manifold
  projection.
\newblock In \emph{International conference on machine learning}, pp.\
  3530--3540. PMLR, 2021.

\bibitem[Gardiner(1985)]{gardiner1985handbook}
Crispin~W Gardiner.
\newblock Handbook of stochastic methods for physics, chemistry and the natural
  sciences.
\newblock \emph{Springer series in synergetics}, 1985.

\bibitem[Geiger \& Smidt(2022)Geiger and Smidt]{geiger2022e3nn}
Mario Geiger and Tess Smidt.
\newblock e3nn: Euclidean neural networks.
\newblock \emph{arXiv preprint arXiv:2207.09453}, 2022.

\bibitem[Haas et~al.(2024)Haas, Huberman-Spiegelglas, Mulayoff, Gra{\ss}hof,
  Brandt, and Michaeli]{haas2024discovering}
Ren{\'e} Haas, Inbar Huberman-Spiegelglas, Rotem Mulayoff, Stella Gra{\ss}hof,
  Sami~S Brandt, and Tomer Michaeli.
\newblock Discovering interpretable directions in the semantic latent space of
  diffusion models.
\newblock In \emph{2024 IEEE 18th International Conference on Automatic Face
  and Gesture Recognition (FG)}, pp.\  1--9. IEEE, 2024.

\bibitem[Ho et~al.(2020)Ho, Jain, and Abbeel]{ho2020denoising}
Jonathan Ho, Ajay Jain, and Pieter Abbeel.
\newblock Denoising diffusion probabilistic models.
\newblock \emph{Advances in neural information processing systems},
  33:\penalty0 6840--6851, 2020.

\bibitem[Huang et~al.(2021)Huang, Lim, and Courville]{huang2021variational}
Chin-Wei Huang, Jae~Hyun Lim, and Aaron~C Courville.
\newblock A variational perspective on diffusion-based generative models and
  score matching.
\newblock \emph{Advances in Neural Information Processing Systems},
  34:\penalty0 22863--22876, 2021.

\bibitem[Huang et~al.(2022)Huang, Aghajohari, Bose, Panangaden, and
  Courville]{huang2022riemannian}
Chin-Wei Huang, Milad Aghajohari, Joey Bose, Prakash Panangaden, and Aaron~C
  Courville.
\newblock Riemannian diffusion models.
\newblock \emph{Advances in Neural Information Processing Systems},
  35:\penalty0 2750--2761, 2022.

\bibitem[Hutchinson et~al.(2021)Hutchinson, Le~Lan, Zaidi, Dupont, Teh, and
  Kim]{hutchinson2021lietransformer}
Michael~J Hutchinson, Charline Le~Lan, Sheheryar Zaidi, Emilien Dupont,
  Yee~Whye Teh, and Hyunjik Kim.
\newblock Lietransformer: Equivariant self-attention for lie groups.
\newblock In \emph{International Conference on Machine Learning}, pp.\
  4533--4543. PMLR, 2021.

\bibitem[Hyv{\"a}rinen \& Dayan(2005)Hyv{\"a}rinen and
  Dayan]{hyvarinen2005estimation}
Aapo Hyv{\"a}rinen and Peter Dayan.
\newblock Estimation of non-normalized statistical models by score matching.
\newblock \emph{Journal of Machine Learning Research}, 6\penalty0 (4), 2005.

\bibitem[Jing et~al.(2022)Jing, Corso, Chang, Barzilay, and
  Jaakkola]{jing2022torsional}
Bowen Jing, Gabriele Corso, Jeffrey Chang, Regina Barzilay, and Tommi Jaakkola.
\newblock Torsional diffusion for molecular conformer generation.
\newblock \emph{Advances in Neural Information Processing Systems},
  35:\penalty0 24240--24253, 2022.

\bibitem[Kac \& Kac(1983)Kac and Kac]{kac1983invariant}
Victor~G Kac and Victor~G Kac.
\newblock The invariant bilinear form and the generalized casimir operator.
\newblock \emph{Infinite Dimensional Lie Algebras: An Introduction}, pp.\
  14--24, 1983.

\bibitem[Karpatne et~al.(2018)Karpatne, Ebert-Uphoff, Ravela, Babaie, and
  Kumar]{karpatne2018machine}
Anuj Karpatne, Imme Ebert-Uphoff, Sai Ravela, Hassan~Ali Babaie, and Vipin
  Kumar.
\newblock Machine learning for the geosciences: Challenges and opportunities.
\newblock \emph{IEEE Transactions on Knowledge and Data Engineering},
  31\penalty0 (8):\penalty0 1544--1554, 2018.

\bibitem[Ketata et~al.(2023)Ketata, Laue, Mammadov, St{\"a}rk, Wu, Corso,
  Marquet, Barzilay, and Jaakkola]{ketata2023diffdock}
Mohamed~Amine Ketata, Cedrik Laue, Ruslan Mammadov, Hannes St{\"a}rk, Menghua
  Wu, Gabriele Corso, C{\'e}line Marquet, Regina Barzilay, and Tommi~S
  Jaakkola.
\newblock Diffdock-pp: Rigid protein-protein docking with diffusion models.
\newblock \emph{arXiv preprint arXiv:2304.03889}, 2023.

\bibitem[Kim et~al.(2022)Kim, Na, Kwon, Lee, Kang, and Moon]{kim2022maximum}
Dongjun Kim, Byeonghu Na, Se~Jung Kwon, Dongsoo Lee, Wanmo Kang, and Il-Chul
  Moon.
\newblock Maximum likelihood training of implicit nonlinear diffusion model.
\newblock \emph{Advances in neural information processing systems},
  35:\penalty0 32270--32284, 2022.

\bibitem[Klimovskaia et~al.(2020)Klimovskaia, Lopez-Paz, Bottou, and
  Nickel]{klimovskaia2020poincare}
Anna Klimovskaia, David Lopez-Paz, L{\'e}on Bottou, and Maximilian Nickel.
\newblock Poincar{\'e} maps for analyzing complex hierarchies in single-cell
  data.
\newblock \emph{Nature communications}, 11\penalty0 (1):\penalty0 2966, 2020.

\bibitem[Kong \& Tao(2024)Kong and Tao]{kong2024convergence}
Lingkai Kong and Molei Tao.
\newblock Convergence of kinetic langevin monte carlo on lie groups.
\newblock \emph{arXiv preprint arXiv:2403.12012}, 2024.

\bibitem[Le et~al.(2021)Le, Bertolini, No{\'e}, and
  Clevert]{le2021parameterized}
Tuan Le, Marco Bertolini, Frank No{\'e}, and Djork-Arn{\'e} Clevert.
\newblock Parameterized hypercomplex graph neural networks for graph
  classification.
\newblock In \emph{International Conference on Artificial Neural Networks},
  pp.\  204--216. Springer, 2021.

\bibitem[Le et~al.(2022{\natexlab{a}})Le, No{\'e}, and
  Clevert]{le2022equivariant}
Tuan Le, Frank No{\'e}, and Djork-Arn{\'e} Clevert.
\newblock Equivariant graph attention networks for molecular property
  prediction.
\newblock \emph{arXiv preprint arXiv:2202.09891}, 2022{\natexlab{a}}.

\bibitem[Le et~al.(2022{\natexlab{b}})Le, Noe, and
  Clevert]{le2022representation}
Tuan Le, Frank Noe, and Djork-Arn{\'e} Clevert.
\newblock Representation learning on biomolecular structures using equivariant
  graph attention.
\newblock In \emph{The First Learning on Graphs Conference},
  2022{\natexlab{b}}.
\newblock URL \url{https://openreview.net/forum?id=kv4xUo5Pu6}.

\bibitem[Leach et~al.(2022)Leach, Schmon, Degiacomi, and
  Willcocks]{leach2022denoising}
Adam Leach, Sebastian~M Schmon, Matteo~T. Degiacomi, and Chris~G. Willcocks.
\newblock Denoising diffusion probabilistic models on {SO}(3) for rotational
  alignment.
\newblock In \emph{ICLR 2022 Workshop on Geometrical and Topological
  Representation Learning}, 2022.
\newblock URL \url{https://openreview.net/forum?id=BY88eBbkpe5}.

\bibitem[Li et~al.(2023)Li, Xue, Liu, and Lai]{Li_2023_CVPR}
Bo~Li, Kaitao Xue, Bin Liu, and Yu-Kun Lai.
\newblock Bbdm: Image-to-image translation with brownian bridge diffusion
  models.
\newblock In \emph{Proceedings of the IEEE/CVF Conference on Computer Vision
  and Pattern Recognition (CVPR)}, pp.\  1952--1961, June 2023.

\bibitem[Liao \& Liu(2023)Liao and Liu]{liao2023lie}
Dengfeng Liao and Guangzhong Liu.
\newblock Lie group equivariant convolutional neural network based on laplace
  distribution.
\newblock \emph{Remote Sensing}, 15\penalty0 (15):\penalty0 3758, 2023.

\bibitem[Lin et~al.(2016)Lin, Drton, and Shojaie]{lin2016estimation}
Lina Lin, Mathias Drton, and Ali Shojaie.
\newblock Estimation of high-dimensional graphical models using regularized
  score matching.
\newblock \emph{Electronic journal of statistics}, 10\penalty0 (1):\penalty0
  806, 2016.

\bibitem[Lipman et~al.(2022)Lipman, Chen, Ben-Hamu, Nickel, and
  Le]{lipman2022flow}
Yaron Lipman, Ricky~TQ Chen, Heli Ben-Hamu, Maximilian Nickel, and Matt Le.
\newblock Flow matching for generative modeling.
\newblock \emph{arXiv preprint arXiv:2210.02747}, 2022.

\bibitem[Lyu(2009)]{lyu2009interpretation}
Siwei Lyu.
\newblock Interpretation and generalization of score matching.
\newblock In \emph{Proceedings of the Twenty-Fifth Conference on Uncertainty in
  Artificial Intelligence}, pp.\  359--366, 2009.

\bibitem[Pang et~al.(2020)Pang, Xu, Li, Song, Ermon, and
  Zhu]{pang2020efficient}
Tianyu Pang, Kun Xu, Chongxuan Li, Yang Song, Stefano Ermon, and Jun Zhu.
\newblock Efficient learning of generative models via finite-difference score
  matching.
\newblock \emph{Advances in Neural Information Processing Systems},
  33:\penalty0 19175--19188, 2020.

\bibitem[Park et~al.(2023)Park, Kwon, Choi, Jo, and Uh]{park2023understanding}
Yong-Hyun Park, Mingi Kwon, Jaewoong Choi, Junghyo Jo, and Youngjung Uh.
\newblock Understanding the latent space of diffusion models through the lens
  of riemannian geometry.
\newblock \emph{Advances in Neural Information Processing Systems},
  36:\penalty0 24129--24142, 2023.

\bibitem[Peluchetti(2023)]{peluchetti2023non}
Stefano Peluchetti.
\newblock Non-denoising forward-time diffusions.
\newblock \emph{arXiv preprint arXiv:2312.14589}, 2023.

\bibitem[Pfau et~al.(2020)Pfau, Higgins, Botev, and
  Racani{\`e}re]{pfau2020disentangling}
David Pfau, Irina Higgins, Alex Botev, and S{\'e}bastien Racani{\`e}re.
\newblock Disentangling by subspace diffusion.
\newblock \emph{Advances in Neural Information Processing Systems},
  33:\penalty0 17403--17415, 2020.

\bibitem[Puny et~al.(2021)Puny, Atzmon, Ben-Hamu, Misra, Grover, Smith, and
  Lipman]{puny2021frame}
Omri Puny, Matan Atzmon, Heli Ben-Hamu, Ishan Misra, Aditya Grover, Edward~J
  Smith, and Yaron Lipman.
\newblock Frame averaging for invariant and equivariant network design.
\newblock \emph{arXiv preprint arXiv:2110.03336}, 2021.

\bibitem[Rackers et~al.(2023)Rackers, Tecot, Geiger, and Smidt]{Rackers_2023}
Joshua~A Rackers, Lucas Tecot, Mario Geiger, and Tess~E Smidt.
\newblock A recipe for cracking the quantum scaling limit with machine learned
  electron densities.
\newblock \emph{Machine Learning: Science and Technology}, 4\penalty0
  (1):\penalty0 015027, feb 2023.
\newblock \doi{10.1088/2632-2153/acb314}.
\newblock URL \url{https://dx.doi.org/10.1088/2632-2153/acb314}.

\bibitem[Ramakrishnan et~al.(2014)Ramakrishnan, Dral, Rupp, and von
  Lilienfeld]{Ramakrishnan2014}
Raghunathan Ramakrishnan, Pavlo~O. Dral, Matthias Rupp, and O.~Anatole von
  Lilienfeld.
\newblock Quantum chemistry structures and properties of 134 kilo molecules.
\newblock \emph{Scientific Data}, 1\penalty0 (1):\penalty0 140022, Aug 2014.
\newblock ISSN 2052-4463.
\newblock \doi{10.1038/sdata.2014.22}.
\newblock URL \url{https://doi.org/10.1038/sdata.2014.22}.

\bibitem[Rappe et~al.(1992)Rappe, Casewit, Colwell, Goddard, and
  Skiff]{uff_rappe_92}
A.~K. Rappe, C.~J. Casewit, K.~S. Colwell, W.~A.~III Goddard, and W.~M. Skiff.
\newblock Uff, a full periodic table force field for molecular mechanics and
  molecular dynamics simulations.
\newblock \emph{Journal of the American Chemical Society}, 114\penalty0
  (25):\penalty0 10024--10035, 1992.
\newblock \doi{10.1021/ja00051a040}.
\newblock URL \url{https://doi.org/10.1021/ja00051a040}.

\bibitem[Romero et~al.(2020)Romero, Bekkers, Tomczak, and
  Hoogendoorn]{romero2020attentive}
David Romero, Erik Bekkers, Jakub Tomczak, and Mark Hoogendoorn.
\newblock Attentive group equivariant convolutional networks.
\newblock In \emph{International Conference on Machine Learning}, pp.\
  8188--8199. PMLR, 2020.

\bibitem[Romero \& Cordonnier(2020)Romero and Cordonnier]{romero2020group}
David~W Romero and Jean-Baptiste Cordonnier.
\newblock Group equivariant stand-alone self-attention for vision.
\newblock \emph{arXiv preprint arXiv:2010.00977}, 2020.

\bibitem[S{\"a}rkk{\"a} \& Solin(2019)S{\"a}rkk{\"a} and
  Solin]{sarkka2019applied}
Simo S{\"a}rkk{\"a} and Arno Solin.
\newblock \emph{Applied stochastic differential equations}, volume~10.
\newblock Cambridge University Press, 2019.

\bibitem[Satorras et~al.(2021)Satorras, Hoogeboom, and Welling]{satorras2021n}
V{\i}ctor~Garcia Satorras, Emiel Hoogeboom, and Max Welling.
\newblock E (n) equivariant graph neural networks.
\newblock In \emph{International conference on machine learning}, pp.\
  9323--9332. PMLR, 2021.

\bibitem[Sch{\"u}tt et~al.(2021)Sch{\"u}tt, Unke, and
  Gastegger]{schutt2021equivariant}
Kristof Sch{\"u}tt, Oliver Unke, and Michael Gastegger.
\newblock Equivariant message passing for the prediction of tensorial
  properties and molecular spectra.
\newblock In \emph{International Conference on Machine Learning}, pp.\
  9377--9388. PMLR, 2021.

\bibitem[Smidt(2021)]{smidt2021euclidean}
Tess~E Smidt.
\newblock Euclidean symmetry and equivariance in machine learning.
\newblock \emph{Trends in Chemistry}, 3\penalty0 (2):\penalty0 82--85, 2021.

\bibitem[Sohl-Dickstein et~al.(2015)Sohl-Dickstein, Weiss, Maheswaranathan, and
  Ganguli]{sohl2015deep}
Jascha Sohl-Dickstein, Eric Weiss, Niru Maheswaranathan, and Surya Ganguli.
\newblock Deep unsupervised learning using nonequilibrium thermodynamics.
\newblock In \emph{International conference on machine learning}, pp.\
  2256--2265. PMLR, 2015.

\bibitem[Song et~al.(2020{\natexlab{a}})Song, Garg, Shi, and
  Ermon]{song2020sliced}
Yang Song, Sahaj Garg, Jiaxin Shi, and Stefano Ermon.
\newblock Sliced score matching: A scalable approach to density and score
  estimation.
\newblock In \emph{Uncertainty in Artificial Intelligence}, pp.\  574--584.
  PMLR, 2020{\natexlab{a}}.

\bibitem[Song et~al.(2020{\natexlab{b}})Song, Sohl-Dickstein, Kingma, Kumar,
  Ermon, and Poole]{song2020score}
Yang Song, Jascha Sohl-Dickstein, Diederik~P Kingma, Abhishek Kumar, Stefano
  Ermon, and Ben Poole.
\newblock Score-based generative modeling through stochastic differential
  equations.
\newblock In \emph{International Conference on Learning Representations},
  2020{\natexlab{b}}.

\bibitem[Song et~al.(2021)Song, Durkan, Murray, and Ermon]{song2021maximum}
Yang Song, Conor Durkan, Iain Murray, and Stefano Ermon.
\newblock Maximum likelihood training of score-based diffusion models.
\newblock \emph{Advances in neural information processing systems},
  34:\penalty0 1415--1428, 2021.

\bibitem[Thomas et~al.(2018)Thomas, Smidt, Kearnes, Yang, Li, Kohlhoff, and
  Riley]{thomas2018tensor}
Nathaniel Thomas, Tess Smidt, Steven Kearnes, Lusann Yang, Li~Li, Kai Kohlhoff,
  and Patrick Riley.
\newblock Tensor field networks: Rotation-and translation-equivariant neural
  networks for 3d point clouds.
\newblock \emph{arXiv preprint arXiv:1802.08219}, 2018.

\bibitem[Wang et~al.(2023)Wang, Schiff, Gokaslan, Pan, Wang, De~Sa, and
  Kuleshov]{wang2023infodiffusion}
Yingheng Wang, Yair Schiff, Aaron Gokaslan, Weishen Pan, Fei Wang, Christopher
  De~Sa, and Volodymyr Kuleshov.
\newblock Infodiffusion: Representation learning using information maximizing
  diffusion models.
\newblock In \emph{International Conference on Machine Learning}, pp.\
  36336--36354. PMLR, 2023.

\bibitem[Weiler \& Cesa(2019)Weiler and Cesa]{weiler2019general}
Maurice Weiler and Gabriele Cesa.
\newblock General e (2)-equivariant steerable cnns.
\newblock \emph{Advances in neural information processing systems}, 32, 2019.

\bibitem[Weiler et~al.(2018)Weiler, Geiger, Welling, Boomsma, and
  Cohen]{weiler20183d}
Maurice Weiler, Mario Geiger, Max Welling, Wouter Boomsma, and Taco~S Cohen.
\newblock 3d steerable cnns: Learning rotationally equivariant features in
  volumetric data.
\newblock \emph{Advances in Neural Information Processing Systems}, 31, 2018.

\bibitem[Winter et~al.(2022)Winter, Bertolini, Le, No{\'e}, and
  Clevert]{winter2022unsupervised}
Robin Winter, Marco Bertolini, Tuan Le, Frank No{\'e}, and Djork-Arn{\'e}
  Clevert.
\newblock Unsupervised learning of group invariant and equivariant
  representations.
\newblock \emph{Advances in Neural Information Processing Systems},
  35:\penalty0 31942--31956, 2022.

\bibitem[Xu et~al.(2022)Xu, Yu, Song, Shi, Ermon, and Tang]{xu2022geodiff}
Minkai Xu, Lantao Yu, Yang Song, Chence Shi, Stefano Ermon, and Jian Tang.
\newblock Geodiff: A geometric diffusion model for molecular conformation
  generation.
\newblock \emph{arXiv preprint arXiv:2203.02923}, 2022.

\bibitem[Yim et~al.(2023)Yim, Trippe, De~Bortoli, Mathieu, Doucet, Barzilay,
  and Jaakkola]{pmlr-v202-yim23a}
Jason Yim, Brian~L. Trippe, Valentin De~Bortoli, Emile Mathieu, Arnaud Doucet,
  Regina Barzilay, and Tommi Jaakkola.
\newblock {SE}(3) diffusion model with application to protein backbone
  generation.
\newblock In Andreas Krause, Emma Brunskill, Kyunghyun Cho, Barbara Engelhardt,
  Sivan Sabato, and Jonathan Scarlett (eds.), \emph{Proceedings of the 40th
  International Conference on Machine Learning}, volume 202 of
  \emph{Proceedings of Machine Learning Research}, pp.\  40001--40039. PMLR,
  23--29 Jul 2023.
\newblock URL \url{https://proceedings.mlr.press/v202/yim23a.html}.

\bibitem[Zhang et~al.(2024)Zhang, Zhang, Zhong, Misra, and
  Tang]{zhang2024diffpack}
Yangtian Zhang, Zuobai Zhang, Bozitao Zhong, Sanchit Misra, and Jian Tang.
\newblock Diffpack: A torsional diffusion model for autoregressive protein
  side-chain packing.
\newblock \emph{Advances in Neural Information Processing Systems}, 36, 2024.

\bibitem[Zhu et~al.(2024)Zhu, Chen, Kong, Theodorou, and
  Tao]{zhu2024trivialized}
Yuchen Zhu, Tianrong Chen, Lingkai Kong, Evangelos~A Theodorou, and Molei Tao.
\newblock Trivialized momentum facilitates diffusion generative modeling on lie
  groups.
\newblock \emph{arXiv preprint arXiv:2405.16381}, 2024.

\end{thebibliography}
\bibliographystyle{iclr2025_conference}

\newpage

\onecolumn
\appendix

\section{Appendix A: Summary of Notation and Intuition}

\begin{table}[h]
\centering
\renewcommand{\arraystretch}{1.3}
\begin{tabular}{p{1.1cm} p{4.0cm} p{4.0cm} p{5.5cm}}
\toprule
\textbf{Symbol} & \textbf{Name} & \textbf{Definition} & \textbf{Intuition} \\
\midrule
$G$ & Lie group & & A continuous symmetry group, e.g., rotations ($\mathrm{SO}(3)$), translations, scalings. Encodes the structure of transformations acting on the data. \\
$e$ & Identity element of $G$ & $e \cdot g  = g \forall g\in G$ & The identity transformation leaving everything unchanged. \\
$\mathfrak{g}$ & Lie algebra of $G$ & $T_eG$ & Tangent space at the identity; represents infinitesimal group transformations. \\
$X$ & Data manifold & & The space where the data lives, often $\mathbb{R}^n$, but can be more general or even discrete (e.g., graph for molecules, grid for images, etc.). \\
$\rho_X$ & Group action & $\rho_X: G \times X \to X$ & Specifies how each abstract group element $g \in G$ transforms data points in $X$ via matrices. \\
$G \cdot \bx$ & Orbit of $\bx$ under $G$ & $ \{ \rho_X(g)(\bx),~ g\in G \}$ & The set of all points reachable from $x$ via group actions. Captures the “symmetry class” of $x$. \\
$G_\bx$ & Stabilizer subgroup at $\bx$ & $ \{ g\in G | \rho_X(g)(\bx) = \bx \}$& Subgroup of $G$ that leaves $\bx$ unchanged. Describes residual symmetries at that point. \\
$d\rho_X$ & Infinitesimal action & $d\rho_X : \mathfrak{g} \to \text{Vect}(X)$ & Maps infinitesimal transformations to vector fields on $X$; captures how a tiny "step" in $G$ moves a point in $X$. \\
$\exp$ & Exponential map & $\exp(A) = \gamma_A(1)$, where $\gamma_A \colon \mathbb R \to G$ & 
Geodesic path on $G$ determined by the direction given by the vector $A\in \mathfrak{g}$.\\
$\xi_A$ & Flow on $X$ induced by $\rho_X$ & $\rho_X\left(\exp(\tau A)\right)(\bx)$ & Path on $X$ corresponding to a geodesic path on $G$ determined by $A$.\\
$\Pi_A(\bx)$ & Fundamental vector field from $A \in \mathfrak{g}$ & $\frac{d}{d\tau}\big|_{\tau=0} \rho_X (\exp(\tau A)) (\bx)$&  A vector field on $X$ generated by a direction $A$ in the Lie algebra; describes how $x$ moves under an infinitesimal group transformation. \\
\bottomrule
\end{tabular}
\caption{Summary of Lie group/Lie algebra related quantities with their notation, definition and intuitive meaning.}
\end{table}

\section{Examples of Lie groups and Lie algebra actions}

In this appendix we list some important Lie groups and Lie algebra actions, their corresponding fundamental vector fields as well as the fundamental flow coordinates. These will be useful in the main text.

\subsection{\texorpdfstring{$T(N)$}{T(N)}}
\label{app:TN}
Let $ X = \R^N $ and $ G = T(N) $, the group of translations in $N$-dimensional space.
Element of $T(N) $ are represented by a vector 
$\mathbf{v} = (v_1, v_2, \dots, v_N)^\top \in \mathbb{R}^N $, where $ v_i $ are the translation components along the $ x_i $ axes for $ i = 1, \dots, N $, thus $T(N)\simeq \R^N$.
Explicitly, for a $\bx\in X$ its action is given by
$\rho_{\mathbb{R}^N}(\mathbf{v}, \mathbf{x}) = \mathbf{x} + \mathbf{v}$.

The corresponding Lie algebra $ \mathfrak{t}(N)$ is also isomorphic to $\R^N$, and it consists of vectors $ \mathbf{a} = (a_1, a_2, \dots, a_N)^\top \in \mathbb{R}^N $. The Lie bracket of any two elements in $ \mathfrak{t}(N) $ vanishes, as $T(N)$ is Abelian.

To derive the infinitesimal action, we first note that the exponential map is trivial, $\exp(\tau \mathbf{A}) = \tau \mathbf{A}$.
Hence, we have 
\begin{align}
\Pi_A(\mathbf{x}) &= \frac{d}{d\tau}\bigg|_{\tau=0} \rho_{\mathbb{R}^N}(\tau \mathbf{A}, \mathbf{x}) = \frac{d}{d\tau}\bigg|_{\tau=0} \left(\mathbf{x} + \tau \mathbf{A}\right) = \bA~.
\end{align}
Thus, the fundamental vector field $ \Pi_A $ corresponding to $ \mathbf{A} \in \mathfrak{t}(N) $ is the constant vector field:
$$
\Pi_A = a_1 \frac{\partial}{\partial x_1} + a_2 \frac{\partial}{\partial x_2} + \dots + a_N \frac{\partial}{\partial x_N} = \bA \cdot \nabla~.
$$

\subsection{\texorpdfstring{$X = \mathbb{R}^N$, $G = \mathbb{R}_+^*$}{X = RN, G = R+*} (group of dilations)}

Let us consider $ X = \mathbb{R}^N $ and $ G = \mathbb{R}_+^* $, the group of dilations in $ N $-dimensional space.
The group $ \mathbb{R}_+^* $ consists of all positive scaling factors. Each element of $ G = \mathbb{R}_+^* $ can be represented by a scalar $ \lambda > 0 $ that scales all vectors in $ \mathbb{R}^N $ by this factor.

The action of $ G = \mathbb{R}_+^* $ on $ \mathbb{R}^N $ is a dilation, meaning that every vector $ \mathbf{x} = (x_1, x_2, \dots, x_N)^\top \in \mathbb{R}^N $ is scaled by the factor $ \lambda $. Explicitly, the group action is given by
\begin{align}
\rho_{\mathbb{R}^N}(\lambda, \mathbf{x}) = \lambda \mathbf{x}~.
\end{align}

The Lie algebra $ \mathfrak{g} = \mathbb{R} $ corresponding to the dilation group $ G = \mathbb{R}_+^* $ consists of real numbers representing the logarithm of the scaling factor. Specifically, an element $ A \in \mathfrak{g} $ corresponds to a generator of the dilation, and the exponential map $ \exp: \mathfrak{g} \rightarrow G $ is given by:$
\exp(\tau A) = e^{\tau A}$,
where $ \tau $ is a real parameter.

The infinitesimal action corresponds to taking the derivative at $ \tau = 0 $. For a vector $ \mathbf{x} \in \mathbb{R}^N $ and $ A \in \mathfrak{g} $, the fundamental vector field $ \Pi_A $ is computed as:

\begin{align}
d\rho_{\R^N}(A) &= \frac{d}{d\tau}\bigg|_{\tau=0} \rho_{\mathbb{R}^N}(e^{\tau A}, \mathbf{x}) 
= \frac{d}{d\tau}\bigg|_{\tau=0} \left(e^{\tau A} \mathbf{x}\right) = 
A\bx~,
\end{align}
and
$$
\cL_A(\mathbf{x}) = A \bx \cdot \nabla~.
$$
Now, solving the equation
\begin{align}
    \bx = e^{\tau A}\bx_0
\end{align}
in terms of $\tau$ we obtain
\begin{align}
    \tau = \frac{1}{A}\log \frac{|\bx|^2}{\bx \cdot \bx_0} = \frac{1}{A}\log \frac{|\bx|^2}{|\bx| |\bx_0|} = \frac{1}{A}\log \frac{|\bx|}{|\bx_0|}
    = \frac{1}{2A}\log \frac{|\bx|^2}{|\bx_0|^2}~.
\end{align}
In the usual case of $A=1$ (generator of the Lie algebra), $\bx_0=\frac1{\sqrt{N}}(1,1,\dots,1)^\top$ to be the unit vector we obtain the usual expression as flow coordinate
\begin{align}
    \tau = \frac12\log (x_1^2 + x_2^2 + \cdots + x_N^2)~.
\end{align}

\subsection{\texorpdfstring{$X = \mathbb{R}^3$, $G = \SO(3) \times \R_+^*$}{}}\label{appendix:spherical-so(3)-dilation}

The dilation part is solved in the previous section, so we actually just focus on the action of $\SO(3)$ on $\R^3$. The orbits are given by spheres centered at the origin, and we can decompose the action of $\SO(3)$ by variying the azimuthal or the polar angle defined by a vector $\bx$. Namely, we have the two actions
\begin{align}
    \rho_{\R^3}(\varphi, \bx) &= \begin{pmatrix}
        \cos\varphi & -\sin\varphi & 0 \\
        \sin\varphi & \cos\varphi & 0 \\
        0 & 0 & 1
    \end{pmatrix}
    \begin{pmatrix}
        x\\ y\\ z
    \end{pmatrix}~,\nonumber\\
    \rho_{\R^3}(\theta, \bx) &= \left[I + \sin\theta \begin{pmatrix}
        0 & 0 &\cos\varphi \\
        0 & 0 & \sin\varphi\\
         -\cos\varphi & -\sin\varphi & 0
    \end{pmatrix}\right. \nonumber\\
    & \qquad\qquad\qquad \left.+(1-\cos\theta)
    \begin{pmatrix}
        -\cos^2\varphi & -\cos\varphi\sin\varphi & 0 \\
        -\cos\varphi\sin\varphi &-\sin^2\varphi & 0 \\
        0 & 0 & -1
    \end{pmatrix}
    \right]
    \begin{pmatrix}
        x\\ y\\ z
    \end{pmatrix}~.
\end{align}
If we take the differentials
\begin{align} \label{eq:so(3)-theta_0}
     \left. d\rho_{\R^3}(\varphi, \bx)\right|_{\varphi=0} &=\left. \begin{pmatrix}
        -\sin\varphi & -\cos\varphi & 0 \\
        \cos\varphi & -\sin\varphi & 0 \\
        0 & 0 & 0
    \end{pmatrix}
    \begin{pmatrix}
        x\\ y\\ z
    \end{pmatrix} \right|_{\varphi=0}
    = \begin{pmatrix}
        0 & -1 & 0 \\
        1 & 0 & 0 \\
        0 & 0 & 0
    \end{pmatrix}
    \begin{pmatrix}
        x\\ y\\ z
    \end{pmatrix}=A_z\bx~,\nonumber\\
    d\rho_{\R^3}(\theta, \bx) &= \left[\cos\theta \begin{pmatrix}
        0 & 0 &\cos\varphi \\
        0 & 0 & \sin\varphi\\
         -\cos\varphi & -\sin\varphi & 0
    \end{pmatrix}\right.\nonumber\\
    &\qquad\qquad\qquad \left.-\sin\theta
    \begin{pmatrix}
        \cos^2\varphi & \cos\varphi\sin\varphi & 0 \\
        \cos\varphi\sin\varphi &\sin^2\varphi & 0 \\
        0 & 0 & 1
    \end{pmatrix}
    \right]_{\theta=0}
    \begin{pmatrix}
        x\\ y\\ z
    \end{pmatrix}\nonumber\\
    &= \begin{pmatrix}
        0 & 0 &\cos\varphi \\
        0 & 0 & \sin\varphi\\
         -\cos\varphi & -\sin\varphi & 0
    \end{pmatrix}
    \begin{pmatrix}
        x\\ y\\ z
    \end{pmatrix} = (\cos\varphi A_y - \sin\varphi A_x ) \bx~,
\end{align}
where
\begin{align}\label{eq:so(3)-basis_Ai}
    A_x&=\begin{pmatrix}
        0 & 0 & 0 \\
        0 & 0 & -1 \\
        0 & 1 & 0
    \end{pmatrix}
        &A_y&=\begin{pmatrix}
        0 & 0 & 1 \\
        0 & 0 & 0 \\
        -1 & 0 & 0
    \end{pmatrix}
    &A_z&=\begin{pmatrix}
        0 & -1 & 0 \\
        1 & 0 & 0 \\
        0 & 0 & 0
    \end{pmatrix}
\end{align}
form a basis for $\so(3)$. 
The corresponding differential operators are
\begin{align}
    \cL_\varphi &= x\p_y - y\p_x~, &\cL_\theta &= \frac{1}{\sqrt{x^2+y^2}}\left[zx \p_x + zy\p_y - (x^2 + y^2)\p_z \right]~, 
\end{align}
and it is an easy calculation to show that they commute $[\cL_\varphi, \cL_\theta]=0$~. 
The attentive reader might have noticed that the commutation does not hold at the matrices level. While this is expected, since there is no 2-dimensional commuting subalgebra in $\so(3)$, it is nonetheless quite puzzling since everything works out at the level of differential operators. This reflect the fact that the commutation properties are necessary at the level of the action of $\g$ on $X$, and not necessarily at the Lie algebra level. In this case, however, we can elegantly resolve the puzzle,  we found a matrix representation for the action $d\rho_{\R^3}(\theta)\bx$ which does commute with the $\varphi$ action. To do this we note that we can rewrite
\begin{align}
    \cL_\theta & = \frac{\cos\theta}{\sin\theta}x \p_x + \frac{\cos\theta}{\sin\theta} y \p_y - \frac{\sin\theta}{\cos\theta}z\p_z~,
\end{align}
which corresponds to simultaneous dilations, with different coefficient, in the $z$ axis and $x,y$-plane. The finite action takes the form
\begin{align}
    \widetilde{\rho}_{\R^3}(\theta, \bx) & =\exp\left[ \log\sin\theta \begin{pmatrix}
        1 & 0 & 0 \\
        0 & 1 & 0 \\
        0 & 0 & 0
    \end{pmatrix} + \log\cos\theta \begin{pmatrix}
        0 & 0 & 0 \\
        0 & 0 & 0 \\
        0 & 0 & 1
    \end{pmatrix} \right]~,
\end{align}
and computing the first order term we obtain
\begin{align}\label{eq:so(3)-theta}
    d \widetilde{\rho}_{\R^3}(\theta, \bx) & = \begin{pmatrix}
         \frac{\cos\theta}{\sin\theta} & 0 & 0 \\
        0 &  \frac{\cos\theta}{\sin\theta} & 0 \\
        0 & 0 &  - \frac{\sin\theta}{\cos\theta}
    \end{pmatrix} \bx~.
\end{align}
This matrix is diagonal and it trivially commutes with $A_z$. The price we had to pay to realize a system of commuting matrices is that in $\widetilde{\rho}$ the flow parameter $\theta$ appear non-linearly, thus we traded-off commutativity at the level of the Lie algebra matrices for the linearity of the flow parameters at the group level. We remark that both give rise to the same differential operator on $X$, which is the relevant object for our purposes. 

\subsection{\texorpdfstring{$X=\R^{3N}$}{X=\R{3N}} and global \texorpdfstring{$\SO(3)$}{SO(3)}}
\label{app:global_E3}

\begin{figure}[t]
    \centering
    \includegraphics[width=\textwidth]{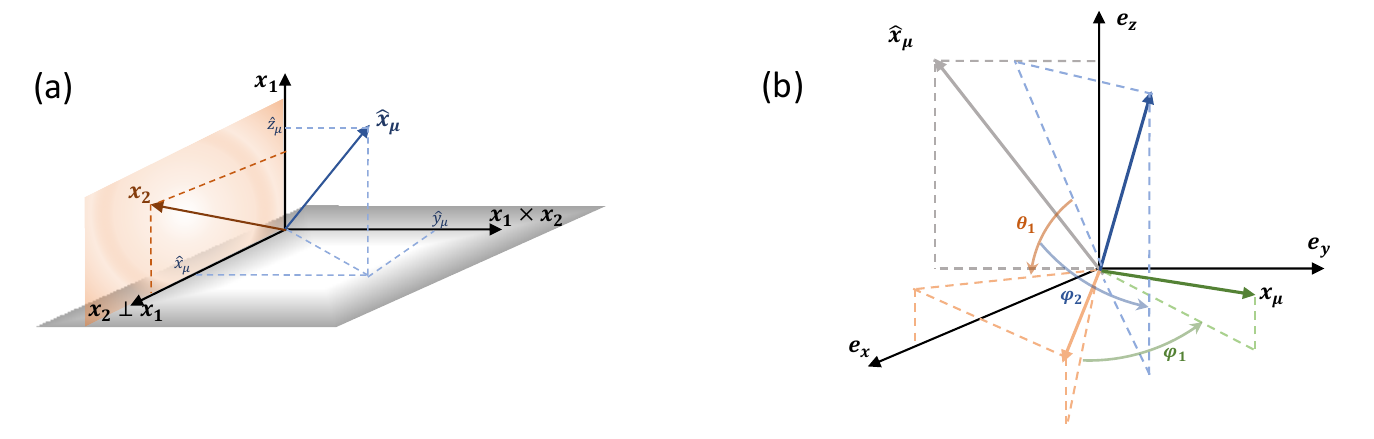}
    \caption{(a) The coordinates $\hbx_\mu$ are the coordinates in the coordinate system defined by $\bx_1$, the orthogonal projection of $\bx_2$ with respect to $x_1$. $\bx_2\perp\bx_1=\bx_2 - \bx_1\cdot \bx_2$, and $\bx_2\times\bx_1$. (b) Graphical depiction of the global symmetry transformations parametrized by the three angles $\varphi_2, \theta_1, \varphi_1$.}
    \label{fig:globalSO3}
\end{figure}

Let $X=\R^{3N}$ be parametrized by $\bx_{i=1,\dots,N}$. We can describe a global $\SO(3)$ action as follows
\begin{align}
    \bx_1 &= R_{\be_z}(\varphi_1) R_{\be_y}(\theta_1) \begin{pmatrix}
        0 \\ 0 \\ \hz_1
    \end{pmatrix}~,\nonumber\\
    \bx_2 &= R_{\be_z}(\varphi_1) R_{\be_y}(\theta_1) R_{\be_z}(\varphi_2) \begin{pmatrix}
    \hx_2 \\ 0 \\ \hz_2
\end{pmatrix}~, \nonumber\\
    \bx_{\mu=3,\dots,N} &= R_{\be_z}(\varphi_1) R_{\be_y}(\theta_1) R_{\be_z}(\varphi_2) \hbx_\mu~,
\end{align}
where $R_{\ba}(\omega)$ represents a rotation of an angle $\omega$ around the axis $\ba$. We can then derive the operator $\bPi \in \R^{3N\times3N}$ as follows. Let $R'(\omega)$ be the matrix where
we take the partial derivative with respect to $\omega$ of all elements of $R$. Then
\begin{align}
    \Pi_{\varphi_1} &=  
    \begin{pmatrix}
       [A_z \bx_1]^\top & [A_z \bx_2]^\top & \cdots & [A_z \bx_N]^\top
    \end{pmatrix}^\top\nonumber\\
    \Pi_{\theta_1} &=\begin{pmatrix}
       (\cos\varphi_1 A_y - \sin\varphi_1 A_x)\bx_1  \\ (\cos\varphi_1 A_y - \sin\varphi_1A_x)\bx_2 \\ \vdots \\ (\cos\varphi_1 A_y - \sin\varphi_1A_x)\bx_N
    \end{pmatrix}\nonumber\\
        \Pi_{\varphi_2} &=\begin{pmatrix}
       \boldsymbol{0} \\ (\sin\theta_1\cos\varphi_1 A_x + \sin\theta_1\sin\varphi_1 A_y+\cos\theta_1 A_z)\bx_2 \\ \vdots \\ (\sin\theta_1\cos\varphi_1 A_x + \sin\theta_1\sin\varphi_1 A_y+\cos\theta_1 A_z)\bx_N
    \end{pmatrix}
\end{align}
Notice that these do represent global rotations since it is easy to see that $(\sin\theta_1\cos\varphi_1 A_x + \sin\theta_1\sin\varphi_1 A_y+\cos\theta_1 A_z)\bx_1 = \boldsymbol{0}$.
Formally, the true  Lie algebra elements are $3\times3$ matrices of the form
\begin{align}
    A_\varphi = \begin{pmatrix}
        A_z & 0 & 0 & \cdots & 0 \\
    0 & A_z & 0 & \cdots & 0 \\
    0 & 0 & A_z & \cdots & 0 \\
    \vdots & \vdots & \vdots & \ddots & \vdots \\
    0 & 0 & 0 & \cdots & A_z
    \end{pmatrix}
\end{align}
and similarly for the other operators. 
Now, for the inverse
relations we have
\begin{align}
\theta_1 &= \arccos \frac{z_1}{(x_1^2+y_1^2+z_1^2)^{1/2}}~,\nonumber\\
\varphi_1 &=\text{sgn}(y_1)  \arccos \frac{x_1}{(x_1^2+y_1^2)^{1/2}}~,\nonumber\\
\varphi_2 &= \arctan \frac{\ty_2}{\tx_2}~,
\end{align}
where $\tbx_2 = R_{\be_y}(\theta_1)^{-1} R_{\be_z}(\varphi_1)^{-1} \bx_2 = R_{\be_y}(-\theta_1) R_{\be_z}(-\varphi_1) \bx_2$.
\subsection{\texorpdfstring{$X=\R^{4}, G=\SO(4)\times \R_+$}{X=R4, G=SO(4) x R+}}
Now we look at the case of a higher dimensional Lie group, namely $G=\SO(4)\times\R_+$. The parametrization is given by
\begin{align}
    x_1 &= e^r \cos\varphi_1 ~, \nonumber\\
    x_2 &= e^r \sin\varphi_1 \cos\varphi_2~, \nonumber\\
    x_3 &= e^r \sin\varphi_1 \sin\varphi_2 \cos\varphi_3~, \nonumber\\
    x_4 &= e^r \sin\varphi_1 \sin\varphi_2\sin\varphi_3~.
\end{align}
The Lie algebra elements corresponding to the $\SO(4)$ flow coordinates are
\begin{align}
\label{eq:sO(3)liematrix}
    A_{\varphi_1} & = \begin{pmatrix}
        0 &-\cos\varphi_2 & -\sin\varphi_2\cos\varphi_3 & -\sin\varphi_2\sin\varphi_3 \\
        \cos\varphi_2 & 0 & 0 & 0 \\
        \sin\varphi_2\cos\varphi_3 & 0 & 0 & 0\\
        \sin\varphi_2\sin\varphi_3 & 0 & 0 & 0
    \end{pmatrix}~, \nonumber\\
    A_{\varphi_2} & = \begin{pmatrix}
        0 &0 & 0 & 0 \\
        0 & 0 & -\cos\varphi_3 & -\sin\varphi_3 \\
        0 & \cos\varphi_3 & 0 & 0\\
        0 & \sin\varphi_3 & 0 & 0
    \end{pmatrix}~, \nonumber\\
    A_{\varphi_3} & = \begin{pmatrix}
        0 &0 & 0 & 0 \\
        0 & 0 & 0 & 0 \\
        0 & 0 & 0 & -1\\
        0 & 0 & 1 & 0
    \end{pmatrix}~. 
\end{align}
Next, we compute the three non-trivial commutators (note that an operators with itself always commute).
First, we list the differential operators 
\begin{align}
\cL_{\varphi_1} &= \frac{1}{\sqrt{x_2^2+x_3^2+x_4^2}} \left[ x_1x_2 \partial_2 + x_1x_3\partial_3 +x_1x_4 \partial_4 - (x_2^2+x_3^2+x_4^2)\partial_1 \right]~,\nonumber\\ 
\cL_{\varphi_2} &= \frac{1}{\sqrt{x_3^2+x_4^2}} \left[ x_2x_3\partial_3 +x_2x_4 \partial_4 - (x_3^2+x_4^2)\partial_2 \right]~,\nonumber\\
\cL_{\varphi_3} &= x_3\partial_4 -x_4 \partial_3~.
\end{align}
where we used the notation $\partial_i = \partial_{x_i}$. These follow directly from \eqref{eq:sO(3)liematrix} together with $L_{\varphi_i} = A_{\varphi_i} \boldsymbol{x} \cdot \nabla$, and using the relations
\begin{align}
\sin\varphi_3= \frac{x_4}{\sqrt{x_3^2+x_4^2}}, \quad \cos\varphi_3= \frac{x_3}{\sqrt{x_3^2+x_4^2}}, \quad \cos\varphi_2= \frac{x_2}{\sqrt{x_2^2+x_3^2+x_4^2}}, \quad \sin\varphi_2= \frac{\sqrt{x_3^2+x_4^2}}{\sqrt{x_2^2+x_3^2+x_4^2}}~.
\end{align}
Specifically, we have
\begin{align}
    [\cL_{\varphi_2}, \cL_{\varphi_3}] &=\frac{1}{\sqrt{x_3^2+x_4^2}} [ -x_2x_4\partial_3 +x_2x_3 \partial_4] - \frac{-x_3x_4+x_4x_3}{({x_3^2+x_4^2})^{1/2}} \left[ x_2x_3\partial_3 +x_2x_4 \partial_4 - (x_3^2+x_4^2)\partial_2 \right] \nonumber\\
    &\quad- \frac{1}{\sqrt{x_3^2+x_4^2}}\left[ x_2x_3\partial_3 -2x_3x_4\partial_2 - x_2x_4\partial_4 +2x_3x_4\partial_1 \right] \nonumber\\ 
    &=\frac{1}{\sqrt{x_2^2+x_3^2+x_4^2}} \left[ x_1x_3\partial_4 -x_1x_4 \partial_3\right] - \frac{1}{\sqrt{x_2^2+x_3^2+x_4^2}}\left[ x_1x_3\partial_4 -2x_3x_4^2\partial_1 - x_1x_4\partial_3 +2x_3x_4\partial_1 \right]\nonumber\\
    &=0~.
    \end{align}
    \begin{align}
    [\cL_{\varphi_1}, \cL_{\varphi_3}] &=\frac{1}{\sqrt{x_2^2+x_3^2+x_4^2}} \left[ x_1x_3\partial_4 -x_1x_4 \partial_3\right]- \frac{-x_3x_4+x_4x_3}{({x_2^2+x_3^2+x_4^2})^{3/2}}\left[ x_1x_2 \partial_2 + x_1x_3\partial_3 +x_1x_4 \partial_4 - (x_2^2+x_3^2+x_4^2)\partial_1 \right]\nonumber\\
    &\quad- \frac{1}{\sqrt{x_2^2+x_3^2+x_4^2}}\left[ x_1x_3\partial_4 -2x_3x_4\partial_1 - x_1x_4\partial_3 +2x_3x_4\partial_1 \right] \nonumber\\ 
    &=\frac{1}{\sqrt{x_2^2+x_3^2+x_4^2}} \left[ x_1x_3\partial_4 -x_1x_4 \partial_3\right] - \frac{1}{\sqrt{x_2^2+x_3^2+x_4^2}}\left[ x_1x_3\partial_4 -2x_3x_4\partial_1 - x_1x_4\partial_3 +2x_3x_4\partial_1 \right]\nonumber\\
    &=0~,
    \end{align}
    \begin{align}
    [\cL_{\varphi_1}, \cL_{\varphi_2}] &= \frac{x_3\partial_3+x_4\partial_4 }{\sqrt{x_2^2+x_3^2+x_4^2}} \frac{x_1x_2}{\sqrt{x_3^2+x_4^2}} +\frac{1}{\sqrt{x_2^2+x_3^2+x_4^2}} \frac{-x_1x_3^2-x_1x_4^2}{(x_3^2+x_4^2)^{3/2}} \left[ x_2x_3\partial_3 +x_2x_4 \partial_4 - (x_3^2+x_4^2)\partial_2 \right]\nonumber\\
    &\quad+\frac{1}{\sqrt{x_2^2+x_3^2+x_4^2}} \frac{1}{(x_3^2+x_4^2)^{1/2}} \left[ x_1x_2x_3\partial_3 -2x_1x_3^2\partial_2-2x_1x_4^2\partial_2+ x_1x_2x_4 \partial_4 \right]\nonumber\\
    &\quad- \frac{1}{(x_3^2+x_4^2)^{1/2}} \frac{-x_2x_4^2-x_2x_3^2+(x_3^2+x_4^2)x_2}{({x_2^2+x_3^2+x_4^2})^{3/2}}\left[ x_1x_2 \partial_2 + x_1x_3\partial_3 +x_1x_4 \partial_4 - (x_2^2+x_3^2+x_4^2)\partial_1 \right] \nonumber\\
    &\quad- \frac{1}{(x_3^2+x_4^2)^{1/2}}\frac{1}{\sqrt{x_2^2+x_3^2+x_4^2}}\left[ x_1x_2x_4\partial_4 -2x_2x_4^2\partial_1 + x_1x_2x_3\partial_3 -2x_2x_3^2\partial_1 -(x_3^2 +x_4^2)(x_1\partial_2 - 2x_2 \partial_1) \right]\nonumber\\ 
    &= \frac{1}{\sqrt{x_2^2+x_3^2+x_4^2}} \frac{-x_1(x_3^2+x_4^2)}{(x_3^2+x_4^2)^{3/2}} \left[ x_2x_3\partial_3 +x_2x_4 \partial_4 - (x_3^2+x_4^2)\partial_2 \right]\nonumber\\
    &\quad+ \frac{1}{\sqrt{x_2^2+x_3^2+x_4^2}} \frac{1}{(x_3^2+x_4^2)^{1/2}} \left[ x_1x_2x_3\partial_3 -2x_1x_3^2\partial_2-2x_1x_4^2\partial_2+ x_1x_2x_4 \partial_4 \right] \nonumber\\
    &\quad- \frac{1}{(x_3^2+x_4^2)^{1/2}}\frac{1}{\sqrt{x_2^2+x_3^2+x_4^2}}\left[ -2x_2x_4^2\partial_1 -2x_2x_3^2\partial_1 -(x_3^2 +x_4^2)(x_1\partial_2 - 2x_2 \partial_1) \right] \nonumber\\
    &= \frac{1}{\sqrt{x_2^2+x_3^2+x_4^2}} \frac{-x_1}{(x_3^2+x_4^2)^{1/2}} \left[ - (x_3^2+x_4^2)\partial_2 \right]+\frac{1}{\sqrt{x_2^2+x_3^2+x_4^2}} \frac{1}{(x_3^2+x_4^2)^{1/2}} \left[ -2x_1x_3^2\partial_2-2x_1x_4^2\partial_2 \right]
    \nonumber\\
    &\quad- \frac{1}{(x_3^2+x_4^2)^{1/2}}\frac{1}{\sqrt{x_2^2+x_3^2+x_4^2}}\left[ -(x_3^2 +x_4^2)x_1\partial_2 \right]\nonumber\\
    &=0~.
    \end{align}

\subsection{\texorpdfstring{$G=\SO(N)$}{G=SO(N)}}

We present here the formalism for the $G=\SO(N)$ for any $N\geq4$. The parametrization is given by \citep{blumenson1960derivation}
\begin{align}
    x_1 &= e^r \cos\varphi_1 ~, \nonumber\\
    x_2 &= e^r \sin\varphi_1 \cos\varphi_2~, \nonumber\\
    x_3 &= e^r \sin\varphi_1 \sin\varphi_2 \cos\varphi_3~, \nonumber\\
    \vdots\nonumber\\
    x_j &= e^r \sin\varphi_1 \sin\varphi_2\sin\varphi_3 \cdots \sin\varphi_{j-1}\cos\varphi_j~,\nonumber\\
    \vdots \nonumber\\
    x_{n-1} &= e^r \sin\varphi_1 \sin\varphi_2\sin\varphi_3 \cdots \sin\varphi_{n-2}\cos\varphi_{n-1}~,\nonumber\\
    x_n &= e^r \sin\varphi_1 \sin\varphi_2\sin\varphi_3 \cdots \sin\varphi_{n-2}\sin\varphi_{n-1}~.
\end{align}
The corresponding Lie algebra elements are given by
\begin{align}
    A_{\varphi_{n-1}} & = \begin{pmatrix}
        0 &\cdots & 0 & 0 \\
        0 & \ddots & \vdots & \vdots \\
        0 & \cdots & 0 & -1\\
        0 & \cdots & 1 & 0
    \end{pmatrix}~,
    A_{\varphi_{n-2}}  = \begin{pmatrix}
        0 &\cdots & 0 & 0 & 0\\
        0 &\ddots &\vdots & \vdots & \vdots \\
        0 &\cdots & 0 & -\cos\varphi_{n-1} & -\sin\varphi_{n-1} \\
        0 &\cdots & \cos\varphi_{n-1} & 0 & 0\\
        0 &\cdots & \sin\varphi_{n-1} & 0 & 0
    \end{pmatrix}~, \nonumber\\
    A_{\varphi_{n-3}} & = \begin{pmatrix}
        0 &\cdots  & 0 & 0 & 0 &0\\
        0 &\ddots &\vdots & \vdots & \vdots & \vdots \\
        0 &\cdots  &0 &-\cos\varphi_{n-2} & -\sin\varphi_{n-2}\cos\varphi_{n-1} & -\sin\varphi_{n-2}\sin\varphi_{n-1} \\
        0 &\cdots  &\cos\varphi_{n-2} & 0 & 0 & 0 \\
        0 &\cdots  & \sin\varphi_{n-2}\cos\varphi_{n-1} & 0 & 0 & 0\\
        0 &\cdots  &\sin\varphi_{n-2}\sin\varphi_{n-1} & 0 & 0 & 0
    \end{pmatrix}~,\nonumber\\
    \vdots\nonumber\\
         A_{\varphi_j} & = \frac{1}{x_j}\begin{pmatrix}
        0 & \cdots  & 0 & 0 &0 & \cdots & 0\\
        0 &\ddots &\vdots & \vdots & \vdots & \vdots &\vdots \\
        0 &\cdots &0 &-x_{j+1} &-x_{j+2} & \cdots & -x_n \\
        0 &\cdots &x_{j+1} &0 & 0 & \cdots & 0 \\
        0 &\cdots &x_{j+2} & 0 & 0 & \cdots & 0\\
        0&\cdots  &\vdots &\vdots &\vdots &\ddots &\vdots\\
        0&\cdots  &x_n &0 & 0 & \cdots & 0
    \end{pmatrix}~,\nonumber\\
    \vdots\nonumber\\
     A_{\varphi_{1}} & = \frac{1}{x_1}\begin{pmatrix}
        0 & -x_2  & -x_3 & -x_4 & \cdots & -x_n\\
        x_2 &\ddots &\vdots & \vdots & \vdots & \vdots \\
        x_3 &\cdots  &0 &0 & 0 & 0 \\
        x_4 &\cdots  &0 & 0 & 0 & 0 \\
        \vdots &\cdots  & 0 & 0 & 0 & 0\\
        x_n &\cdots  &0 & 0 & 0 & 0
    \end{pmatrix}~.
\end{align}

\section{Proofs of condition for suitable Lie group}
\label{app:proofs_conditions}

Here we provide the statements with proofs of the results in Section \ref{ss:g_conditions}.
\begin{proposition}
\label{prop:Pi_complete_app}
    The linear operator induced by $\bPi$ is complete if $\bPi$ is the local frame of a vector bundle $E$ over $X$ whose rank is $n\geq \dim X$ almost everywhere. If $\rank\ E = n$ everywhere, then $E=TX$, the tangent bundle of $X$. 
\end{proposition}
\begin{proof}
We start by noting that, given the expression of the fundamental fields as derivations, we can write $\bcL(\bx) = \bPi(\bx)^\top \nabla$. Let $\pi: E\rightarrow X$ be the projection map, then $\rank \pi^{-1}(\bx) = \min (\rank \bPi(\bx), n)$, since $\rank \nabla = n$. 
Now, consider $\bcL \log p(\bx) = \bcL\log q(\bx)$, which implies $\bcL \log \frac{p(\bx)}{q(\bx)}=0$. Let $U\subseteq X$ such that $\rank \bPi \geq n$ $\forall \bx\in U$, and by assumption $X \setminus U$ has measure zero. Then the above holds if and only if
$\nabla \log \frac{p(\bx)}{q(\bx)}=0$, which implies $\frac{p(\bx)}{q(\bx)} = c$, constant $\forall \bx\in U$. Now, $p(\bx)$ and $q(\bx)$ are probability densities by assumption, thus $c=1$, which proves the claim. 
\end{proof}
\begin{proposition}
\label{prop:main_app}
    The operator $\bPi$ induced by $\g$ is complete if and only if 
    the subspace $U\subseteq X$ such that $\dim \frac{G}{G_\bx} < n$ for $\bx\in U$, where $n=\dim X$, has measure zero in $X$.
\end{proposition}
\begin{proof}
First, we recall that the dimension of an orbit $\cO_\bx$ of $\bx \in X$ equals the dimension of the image of the map $d\rho_\bx: \g \rightarrow T_\bx X: \bA \mapsto \bPi(\bx)$. 
   Suppose first that $\bPi$ is complete. Then, from Proposition \ref{prop:Pi_complete_app} the rank of $\bPi(\bx)$ is $\geq n$ almost everywhere, and therefore $\dim G/G_\bx \geq n$ almost everywhere, which implies one direction of the claim.
   The reverse is quite straightforward. Assume that the rank of $\bPi(\bx)$ is $\geq n$ almost everywhere. As $\bPi$ represent the action of the infinitesiamal 
   transformations of $G$, it means that locally $G$ cannot fix points in $X$, thus proving the claim.
\end{proof}

\section{Proof of main theorem}
\label{app:proof_theorem}

Here we provide the full proof of Theorem \ref{th:main_theorem}:
\begin{theorem}
    Let $G$ be a Lie group acting on $X$ satisfying the conditions of Section \ref{ss:g_conditions}, and let $\g$ be its Lie algebra. 
    The pair of SDEs
    \begin{align}
    \label{eq:SDEforward_app}
     d\bx & =
        \left[\beta(t)\bPi(\bx) \bbf(\bx) +\frac{\gamma(t)^2}{2} \rho_X(\Omega) \right]dt  + \gamma(t) \bPi(\bx) d\bW~,\\
        \label{eq:SDEbackward_app}
              d\bx &= \left[ \beta(t)\bPi(\bx)\bbf(\bx)  - \frac{\gamma^2(t)}2 \rho_X(\Omega)
     -\gamma^2(t) \bPi(\bx) \nabla^\top \cdot \bPi(\bx)\right. \nonumber\\
     &\qquad\qquad\qquad\qquad\qquad\qquad \left.- \gamma(t)^2\bPi(\bx)\bcL \log p_t(\bx)\right] 
           dt +\gamma(t) \bPi(\bx)  d\bW ~,
\end{align}
where $\beta,\gamma:\R \rightarrow \R$ are time-dependent functions, $\bPi: \R^n \rightarrow \R^{n\times n}$ the fundamental vector fields,  $\bbf:\R^n \rightarrow \R^n$ the drift, $\Omega = \sum_i A_i^2$ is the quadratic Casimir element of $\g$, and $\bcL=\bPi(\bx)^\top \nabla $ is such that
\begin{enumerate}
    \item The forward-time SDE \eqref{eq:SDEforward_app} is exactly solvable, with solution
    \begin{align}
    \label{eq:SDE_solution_app}
        \bx(t) = \left(\prod_i O_i(\tau_i(t)) \right)\bx(0) = \left( \prod_{i=1}^ne^{\tau_i(t) A_i} \right) \bx(0)~, 
    \end{align}
    where $O_i=e^{\tau_i(t) A_i}$ is the finite group action and $\btau(t)$ is the solution to the SDE
    \begin{align}
    \label{eq:forward-sde-flow-coordinates_app}
        d\btau(\bx) 
        & = \beta(t)\bbf(\bx)dt +\gamma(t) d\bW ~.
    \end{align}
    \item The SDE \eqref{eq:SDEbackward} is the reverse-time process of \eqref{eq:SDEforward}.
    \item The Langevin dynamic of the above SDEs decomposes as a direct sum of $\g$ infinitesimal actions \eqref{eq:fund_vectorfields}, each defining an infinitesimal transformation along the flows $\xi_{\btau}$.
\end{enumerate}
\end{theorem}
\begin{proof}
    We start by proving 3. We start by rewriting \eqref{eq:SDEforward_app} in terms of the fundamental flow coordinates
    $\tau_i=\xi_{A_i}^{-1}(\bx_0)(\bx):X\rightarrow \R$. For this we employ It\^{o}'s Lemma for the multivariate case: given the SDE \eqref{eq:SDEforward_app} and a transformation
    $\btau(\bx)$, it is given by
    \begin{align}
        \label{eq:Itos-lemma}
        d\btau(\bx) & = (\nabla_\bx \btau)^\top\left[\beta(t)\bPi(\bx)\bbf(\bx) + \frac{\gamma^2(t)}2 \rho_X(\Omega) \right] dt + \frac{\gamma^2(t)}2 \Tr\left[\bPi(\bx)^\top \left( H_\bx \tau\right) \bPi(\bx)\right]dt\nonumber\\
        &\quad
        +\gamma(t)(\nabla_\bx \btau)^\top \bPi(\bx) d\bW \nonumber\\
        & = \beta(t)\bbf(\bx) +  \frac{\gamma^2(t)}2\left[ (\nabla_\bx \btau)^\top \Delta_{\btau} \bx + \Tr\left[\bPi(\bx)^\top \left( H_\bx \btau\right) \bPi(\bx)\right]\right] dt
        +\gamma(t) d\bW 
    \end{align}
since $\nabla_\bx \btau = \bPi^{-1}(\bx)$ as matrices. Now, the second term can be rewritten in components as
\begin{align}
    &\left\{(\nabla_\bx \btau)^\top \Delta_{\btau} \bx + \Tr\left[\bPi(\bx)^\top \left( H_\bx \tau\right) \bPi(\bx)\right]\right\}_{k,l}\nonumber\\
    &\qquad\qquad=\sum_{i}\sum_j \frac{\p x_j}{\p \tau_k}\left(\frac{\p}{\p x_j} \frac{\p x_i}{\p \tau_l} \right) \frac{\p \btau}{\p x_i}
    +\sum_i \sum_{j} \frac{\p x_j}{\p \tau_k}\frac{\p x_i}{\p \tau_l} \frac{\p^2 \btau}{\p x_i \p x_j}\nonumber\\
    &\qquad\qquad=\sum_j \frac{\p x_j}{\p \tau_k}\frac{\p}{\p x_j} \left( \sum_i\frac{\p x_i}{\p \tau_l}\frac{\p \btau}{\p x_i} \right)  \nonumber\\
    &\qquad\qquad=\frac{\p }{\p\tau_k} \left( \sum_i\frac{\p x_i}{\p \tau_l}\frac{\p \btau}{\p x_i} \right) \nonumber\\
    &\qquad\qquad=\frac{\p }{\p\tau_k} \frac{\p \btau}{\p \tau_l} \nonumber\\
    &\qquad\qquad= H_{\btau}\btau~,
\end{align}
which vanishes. Thus we proved that 
\begin{align}
\label{eq:tau_SDE}
        d\btau(\bx) 
        & = \beta(t)\bbf(\bx)dt +\gamma(t) d\bW ~,
    \end{align}
and provided that $\bf$ is chosen so that $f_i(\bx(\btau)) = f_i (\tau_i)$, this corresponds to a system of independent SDEs, as claimed.

Now, to prove 1, let $\btau(t)$ be a solution to \eqref{eq:tau_SDE} and $\bx(t)$ as in \eqref{eq:SDE_solution}. Then a Taylor expansion yields
\begin{align}
    \bx(t) = \left[ I + \sum_i \tau_i(t) A_i + \frac{1}{2} \left(\sum_i \tau_i(t) A_i\right)^2 + \cO(\tau_i^3)  \right] \bx(0) 
\end{align}
since $[A_i, A_j]=0$ and where $\cO(\tau_i^3)$ represents terms of third order in $\tau_i$'s. Then taking the differential and dropping higher order terms 
\begin{align}
    d\bx(t) &= \left[ \sum_i d\tau_i(t) A_i + \frac{1}{2} \left(\sum_id \tau_i(t) A_i\right)^2   \right] \bx(0) \nonumber\\
    &=\left[ \sum_i \left[\beta(t)\bbf(\bx)dt +\gamma(t) d\bW \right]A_i + \frac{1}{2} \left(\sum_i \left[\beta(t)\bbf(\bx)dt +\gamma(t) d\bW A_i\right] \right)^2   \right] \bx(0) \nonumber\\
    &=\left[\beta(t)\bPi(\bx)\bbf(\bx)dt +\gamma(t) \bPi(t)d\bW \right] + \frac{1}{2} \left(\sum_i \gamma(t) d\bW A_i \right)^2    \bx(0) \nonumber\\
    &= \left[\beta(t)\bPi(\bx)\bbf(\bx)dt +\gamma(t) \bPi(t)d\bW \right] + \frac{\gamma(t)^2}{2} \left(\sum_i  A_i^2 dt \right)   \bx(0) \nonumber\\
     &= \left[\beta(t)\bPi(\bx) \bbf(\bx) +\frac{\gamma(t)^2}{2} \rho_X(\Omega) \right]dt  + \gamma(t) \bPi(\bx) d\bW~,
\end{align}
which in the forward SDE \eqref{eq:SDEforward}, proving our claim, where we used the relations $dW_i^2=dt$ and $dW_idW_j=0$ for $j\neq i$. 

Finally, we prove 2. To do this it suffices to apply Anderson's result \citep{anderson1982reverse}
\begin{align}
    d\bx &= \left[ \beta(t)\bPi_i(\bx)\bbf(\bx)  + \frac{\gamma^2(t)}2 \rho_X(\Omega)
     -\gamma^2(t)\nabla \cdot (\bPi(\bx)\bPi(\bx)^\top)\right. \nonumber\\
     &\qquad\qquad\qquad\left.- \gamma(t)^2\bPi(\bx)\bPi(\bx)^\top \nabla_\bx \log p_t(\bx)\right] 
           dt +\gamma(t) \bPi(\bx)  d\bW_i ~,
\end{align}
and note that $\bPi(\bx)^\top \nabla_\bx =\bcL$, the generalized score, and
\begin{align}
    \left[\nabla_{\bx} \cdot (\bPi(\bx)\bPi(\bx)^\top)\right]_{i} &= \frac{\p}{\p x_k} \left( \Pi_{ij}\Pi_{kj} \right)\nonumber\\
    &=\frac{\p}{\p x_k} \left( \Pi_{ij}\right) \Pi_{kj}  +   \Pi_{ij}\frac{\p}{\p x_k} \Pi_{kj} \nonumber\\
    &=\frac{\p x_k}{\p \tau_j}\frac{\p}{\p x_k} \left(\frac{\p x_i}{\p \tau_j}\right)   +   \Pi_{ij}[\nabla^\top \cdot \bPi(\bx)]_{j} \nonumber\\
    &=\frac{\p}{\p \tau_j} \left(\frac{\p x_i}{\p \tau_j}\right)   +   \Pi_{ij}[\nabla^\top \cdot \bPi(\bx)]_{j} \nonumber\\  
    &=[\Tr H_{\btau} (\bx) ]_i   +   \Pi_{ij}[\nabla^\top \cdot \bPi(\bx)]_{j} 
\end{align}
where we recall that the divergence of a matrix is a vector whose components are the divergence of its rows. Recalling the relationship between the trace of the Hessian and the Laplacian we can write in operator form
\begin{align}
    \nabla_{\bx} \cdot (\bPi(\bx)\bPi(\bx)^\top) & = \bPi(\bx) \nabla^\top \cdot \bPi(\bx) + \rho_X(\Omega)~,
\end{align}
Plugging this back in into the previous expression we obtain our claim
\begin{align}
    d\bx &= \left[ \beta(t)\bPi(\bx)\bbf(\bx)  - \frac{\gamma^2(t)}2 \rho_X(\Omega)
     -\gamma^2(t) \bPi(\bx) \nabla^\top \cdot \bPi(\bx)\right. \nonumber\\
     &\qquad\qquad\qquad\qquad\qquad\qquad \left.- \gamma(t)^2\bPi(\bx)\bcL \log p_t(\bx)\right] 
           dt +\gamma(t) \bPi(\bx)  d\bW ~.
\end{align}
\end{proof}

\section{Experiments}
\label{app:experiments}

\paragraph{Practical implementation.}
In this section we list practical implementations for training and inference of our proposed Algorithm \ref{alg:training} and Algorithm \ref{alg:sampling} assuming a variance-preserving SDE for the flow-coordinates, see Eq. \eqref{eq:sde-tau-standard},  because we know that this standard SDE is exactly solvable and related to the forward SDE in Cartesian space as stated in the main Theorem \ref{th:main_theorem}.

The implementation showcase the examples for ${G_0}=(\SO(2)\times \R_+)$ (see the second paragraph in \ref{subsection:examples}) for data living in $\bx \in \R^2$ and $G_1=(\SO(3)\times \R_+)$ for $\bx \in \R^3$ from Appendix \ref{appendix:spherical-so(3)-dilation}.

The flow-maps for $G_0$ and $G_1$ can be computed by leveraging the bijection from Cartesian to polar $\btau = (r, \theta)$ for $G_0$ and spherical $\btau= ( r, \theta, \phi)$ for $G_1$, respectively. As stated in the main text and Appendix, we obtain
\begin{align*}
\begin{array}{c}
    M_{G_0}(\bx) = 
    \begin{pmatrix}
        \sqrt{x^2 + y^2} \\
        \arctan(\frac{y}{x})
    \end{pmatrix} \\
\end{array},
\quad \quad
\begin{array}{c}
    M^{-1}_{G_0}(\btau) = 
    \begin{pmatrix}
        r\cos(\theta) \\
        r\sin(\theta)
    \end{pmatrix} \\
\end{array}.
\end{align*}
As mentioned in $\eqref{eq:sde-2d-ob}$, the Lie algebra basis are ${A}_r = \boldsymbol{I}$ and ${A}_\theta=\begin{pmatrix}
    0 & -1\\1 & 0
\end{pmatrix}$, yielding a quadratic Casimir operator  ${A}_r^2 + {A}_\theta^2 = \mathbf{0}$, such that the dynamics induced by the Casimir elements in line 6 in Alg. \ref{alg:sampling} vanishes, i.e. $\mathbf{v}_c = \mathbf{0}$. The dynamics induced by the divergences (line 7 in Alg. \ref{alg:sampling}) returns $\nabla \cdot {A}_r \bx = \nabla \cdot \bx = \sum_{i=1}^2 \frac{\partial}{\partial x_i} x_i = \sum_{i=1}^{2} 1  = 2$ and $\nabla \cdot {A}_\theta \bx = \nabla \cdot (-x_2, x_1)^\top = \frac{\partial}{\partial x_1} (-x_2) + \frac{\partial}{\partial x_2} x_1 = 0$. Therefore, the divergence dynamics returns the velocity component $\mathbf{v}_d = 2{A}_r \bx + 0 {A}_\theta \bx = 2\bx$.

For $G_1$ the bijection to flow- and Cartesian coordinates is well-known as 
\begin{align*}
\begin{array}{c}
    M_{G_1}(\bx) = 
    \begin{pmatrix}
        \sqrt{x^2 + y^2 + z^2} \\
        \arctan(\frac{\sqrt{x^2 + y^2}}{z}) \\
        \arctan(\frac{y}{x})
    \end{pmatrix} \\
\end{array},
\quad \quad
\begin{array}{c}
    M^{-1}_{G_1}(\btau) = 
    \begin{pmatrix}
        r\sin(\theta)\cos(\phi) \\
        r\sin(\theta)\sin(\phi) \\
        r \cos(\theta)
    \end{pmatrix} \\
\end{array}.
\end{align*}

With ${A}_r = \boldsymbol{I}$, ${A}_\theta = (\cos\phi A_y - \sin \phi A_x) $ and $A_\phi = A_z$ as defined in \eqref{eq:so(3)-theta_0}-\eqref{eq:so(3)-basis_Ai}, the quadratic Casimir elements $A_i^2$ are left multiplied with the vector representation $\bx$, we can distinguish each group component as follows
\begin{align*}
    A_r^2 \bx &= \boldsymbol{I}^2\bx  = \boldsymbol{I}\bx = \bx = (x_1, x_2, x_3)^\top \\
    A_\theta^2 \bx &=  ( \cos^2(\phi) A_{y}^2 + \sin^2(\phi)A_{x}^2 - \cos\phi\sin\phi A_yA_x - \cos\phi\sin\phi A_x A_y )   \bx = -\bx = -(x_1, x_2, x_3)^\top \\
    A_\phi^2 \bx &= A_z^2 \bx = -(x_1, x_2, 0)^\top,
\end{align*}
defining the Casimir dynamics in line 6 in Algorithm \ref{alg:sampling}.

The dynamics induced by the divergences are computed in the same manner as shown in the SO(2) examples. Specifically, we obtain the (scalar) divergences
\begin{align*}
    & \nabla \cdot A_r\bx = \nabla \cdot \bx = 3 \\
    & \nabla \cdot A_\phi\bx = \nabla \cdot A_z \bx =  \nabla\cdot(-x_2, x_1, 0) = \frac{\partial}{\partial{x_1}} (-x_2) + \frac{\partial}{\partial x_2}x_1 + \frac{\partial}{\partial x_3} 0 = 0 \\
    & \nabla \cdot A_\theta\bx = \frac{x_3}{\sqrt{x_1^2 + x_2^2}},
\end{align*}
where the last divergence is point dependent.

In practice, it suffices to compute the quadratic Casimir elements directly using GPU-accelerated frameworks when these are point dependent as in $A_\theta$, or pre-compute them should they be constant matrices. The divergences can be computed using automatic differentiation libraries from modern deep learning frameworks.
\paragraph{Experimental details.}
In this final section we present some further details regarding our experiment in Section \ref{s:experiments}. 

\subsection{2D and 3D toy datasets}\label{sec:appendix-toy}

\begin{table}[htbp]
\centering
\caption{Comparison of GSM and Fisher Score matching on 2D and 3D synthetic datasets. Best results are in bold. When numbers are two close we consider them on par. }
\label{tab:method_comparison_appendix_toy}
\begin{tabular}{@{}lcc@{}}
\toprule
\textbf{Dataset} & \textbf{Group} & \textbf{W2} \\
\midrule
MoG (2D) & $\SO(2)\times\mathbb{R}^+$ & 0.34   \\
MoG (2D) & $T(2)$ & $\boldsymbol{0.15}$   \\
\midrule
Concentric Circles (2D) & $\SO(2)\times\mathbb{R}^+$ & $\boldsymbol{0.19}$  \\
Concentric Circles (2D) & $T(2)$ & $\boldsymbol{0.17}$  \\
\midrule
Line (2D) & $\SO(2)\times\mathbb{R}^+$ & $\boldsymbol{0.33}$   \\
Line (2D) & $T(2)$ & 0.56   \\
\midrule
MoG (3D) & $\SO(2)\times\mathbb{R}^+$ & $\boldsymbol{0.40}$  \\
MoG (3D) & $T(3)$ & $\boldsymbol{0.44}$  \\
\midrule
Torus (3D) & $\SO(3)\times\mathbb{R}^+$ & $\mathbf{0.14}$  \\
Torus (3D) & $T(3)$ & 0.35  \\
\midrule
Möbius Strip (3D) & $\SO(3)\times\mathbb{R}^+$ & $\boldsymbol{0.06}$  \\
Möbius Strip (3D) & $T(3)$ & 0.16  \\
\bottomrule
\end{tabular}
\end{table}
We perform a quantitative evaluation using the Wasserstein-2 (W2) distance on synthetic 2D and 3D datasets, comparing standard (Fisher) score matching ($G=T(2)$, $G=T(3)$) with our proposed approach based on Lie groups ($G=\SO(2) \times \mathbb{R}^+$) and ($G=\SO(3) \times \mathbb{R}^+$). A strong bias in such experiments arises from the similarity between the prior and target distributions. The considered toy datasets are often symmetric with respect to the origin in $\mathbb{R}^{2,3}$, as in the standard Gaussian prior in Fisher score matching.
The similarity of the prior distribution to the target one affects decisively the performance of the generating process.

To account for this, we report a normalized W2 metric, dividing the W2 distance between samples and target by the W2 distance between target and the corresponding priors. We observe that generalized score matching (GSM, ours) performs on par or better in most datasets, particularly where symmetry provides a clear inductive bias as indicated in Table \ref{tab:method_comparison_appendix_toy}. In the MoG datasets, standard (Fisher) score matching ($G=T(N)$) outperforms the Lie group model ($G=\SO(N) \times \mathbb{R}^N$), which is expected since no rotational symmetry is present, while translation symmetry effectively helps locate the Gaussian modes. The performance gap becomes even more pronounced in 3D, where GSM shows stronger advantages. We hypothesize that in higher dimensions, memorizing the target distribution becomes more difficult, and models that incorporate symmetry more explicitly benefit increasingly from this inductive bias.

\subsection{MNIST}\label{sec:appendix-mnist}

We parametrize the noising process through the SDE 
\begin{align}\label{eq:mnist-sde}
    d\btau =  \sqrt{\beta(t)}d\bW~,
\end{align}
where we set the drift term to zero. Notice that this choice is consistent with a 
2d-rotation of a function over the grid $\bx_{i,j}$, given by $f(\bx_{i,j})=f_{i,j}$, denoting the value of the pixel of image $f$ at the location $i,j$.
We train a convolutional neural network (CNN) with three convolutional layers followed by fully connected layers that outputs a single value, being the score for the flow coordinate $\tau$. For the specific details of the implementation we refer to the code-base in the SI. 
In sampling, we apply a smoothing function to compensate interpolation artifacts due to rotations on a discretized grid.
We choose $T=100$ time-steps in training but only need $T=10$ time-steps during sampling. 

\subsubsection{BBDM}
We implement the Brownian Bridge Diffusion Model (BBDM) \citep{Li_2023_CVPR} and train it on the rotated MNIST dataset. The BBDM operates on the full pixel space $\R^{784}$ of the $28\times28$ MNIST digits and indicates a continuous time stochastic process conditioned on the starting $\bx(0)$ and end point $\bx(T)$ which are pinned together as paired data. In this case, we assume $\bx(T)\sim p(\bx_T)$ to be a randomly augmented MNIST digit obtained from an original MNIST digit $\bx(0).$ During training, we sample an intermediate point $\bx(t) \sim N(x_t | \mu_t(x(0), x(T)), \Sigma_t)$ where the mean function $\mu_t(t)(\bx(0), \bx(T))$ is a linear interpolation between the endpoints ($\bx(0), \bx(T))$ and use the score-network to predict the original data point $\widehat{\bx}(0) = s_\theta(x_t, t, x_T)$ as opposed to the noise or difference paramterization proposed in the original BBDM paper. We noticed that predicting the original data point led to better sampling quality including the inductive bias that MNIST digits are represented as binary tokens. Furthermore, we observe that the sampling quality is also better when the prior image $x_T$ is input as context into the score network, enforcing a stronger signal throughout the trajectory. As opposed to our model, we trained the BBDM on $T=1000$ diffusion timesteps using the $\sin$-scheduler from BBDM.

To evaluate the quality of generated MNIST samples, we train a convolutional neural network (CNN) classifier to predict digit labels. This classifier provides a reliable metric for assessing the reconstruction accuracy by comparing the predicted labels of original (unrotated) and generated (rotated) images.
The architecture consists of two convolutional blocks, followed by fully connected layers to predict the 10 MNIST digit labels. All convolutional layers use 3x3 kernels with padding 1, and max pooling uses 2x2 kernels. The model is trained using Adam optimizer (lr=0.001), cross-entropy loss, batch size 64, for 10 epochs on the standard MNIST training set (60,000 samples). The trained classifier achieves greater than 99\% accuracy on the MNIST test set, providing a reliable metric for evaluating reconstruction quality.

To calculate the FID scores, we extract the embedding after the second convolutional block.

\subsection{QM9 \& CrossDocked2020}
\paragraph{QM9.}
The conformer generation tasks is about learning a conditional probabilistic map $\bx \sim p_\theta(X|M)$, where $\bx \in \R^{3N}$ for a molecule with $N$ atoms.
We implement a variant of EQGAT \citep{le2022representation} as neural network architecture where input features for the nodes consist of atom types and atomic coordinates, while edge features are encoded to indicate the existence of a single-, double, triple or aromatic bond based on the adjacency matrix. We use $L=5$ message passing layers with $s_\text{dim}=128~, v_\text{dim}=64$ scalar and vector features, respectively. To predict the scores for each atom, we concatenate the hidden scalar and vector embeddings $s\in \R^{128}~, v\in \R^{3\times64}$ into one output embedding $o=\R^{128+3*64}$ which is further processed by a 2-layer MLP with three output units. Notice that the predicted scores per atom are neither invariant nor equivariant since the scalar and vector features are transformed with an MLP.

We choose the drift $f$ with its scaling $\beta$ and the diffusion coefficients $\gamma$ in such way that the forward SDE for the flow coordinates $\btau$ in \eqref{eq:forward-sde-flow-coordinates} has the expression
\begin{align}\label{eq:sde-tau-standard}
    d\btau = -\frac{1}{2}\beta(t) \btau dt + \sqrt{\beta(t)}d\bW~,
\end{align}
where for clarity we have omitted the dependency between the flow coordinates and the original data in Cartesian coordinates, i.e. $\btau(\bx)$, since the coordinate transformations with Lie algebra representation are described in \ref{appendix:spherical-so(3)-dilation}. The forward SDE in \eqref{eq:sde-tau-standard} is commonly known as \textit{variance-preserving} SDE \citep{song2020score}. We use the cosine scheduler proposed by \citet{dhariwal_guidance_cosine} and $T=100$ diffusion timesteps.

\paragraph{CrossDocked.}

For this experiment we adopt again an SDE of the form \eqref{eq:sde-tau-standard} for the three $\SO(3)$ flow coordinates $\theta_1, \varphi_1, \varphi_2$ and the three
$T(3)$ center of mass Cartesian flow coordinates. 
The $\SO(3)$ flow coordinates are always computed and applied in the ligand center of mass. In this way there is no ambiguity regarding the non-commutativity of $\SE(3)$, as rotation around the origin commute with translations of the system. 
We train a variant of EQGAT as in the QM9 case, but now including also node and edge features of the protein pocket. Specifically, the adjacency matrix for the GNN is computed dynamically at each time step, according to the relative distance between ligand and protein. For this, we choose a cut-off of 5 \AA. We also use in this experiment a cosine scheduler and $T=100$ diffusion timesteps.
Since this learning problem is 6-dimensional, we aggregate the last layer's node embeddings from the ligand atoms into a global representation through summation.  This embedding is fed as input into a 2-layer MLP to predict the six scores. 

\subsubsection{RSGM and BBDM n CrossDocked}
\label{app:rsgm}
\paragraph{RSGM}
We utilized the framework of Riemannian Score-Based Generative Models (RSGM) by \citep{de2022riemannian} to model rigid-body motions on $G = (\SO(3) \times T(3))$, in similar fashion to \citep{corso2022diffdock, pmlr-v202-yim23a} by choosing a variance exploding SDE for the rotation dynamics and variance preserving SDE for the global translations. The terminal distribution for the rotation is designed to converge to an isotropic Gaussian distribution on SO(3) \citep{leach2022denoising}, while the terminal distribution for the translation component converges to an isotropic Gaussian in $\R^3$. To obtain the tractable scores for rotation and translation, we use the code by the authors from DiffDock and SE(3)-Diffusion for Protein Backbone Modeling in \url{https://github.com/gcorso/DiffDock/blob/main/utils/so3.py} and \url{https://github.com/jasonkyuyim/se3_diffusion/blob/master/data/se3_diffuser.py} and make sure that the score outputs for rotation and translation are SO(3) equivariant using the same EQGAT model architecture. The (variance-preserving) scheduler for the translation dynamics is chosen in similar fashion to our experiment using the cosine scheduler, while the (variance-exploding) scheduler for the rotation dynamics is implemented as an linear increasing sequence in $\log_{10}$ space with $\sigma_{min} = 0.001$ and $\sigma_{max}=2.0$ and $T=100$ discretized diffusion steps as $\sigma(t) = 10^{t}$ for $t \in (\log_{10}(\sigma_{min}), \log_{10}({\sigma_{max}}))$.

\paragraph{BBDM} In similar fashion to the MNIST experiment, we train and evaluate a standard Euclidean diffusion model on CrossDocked2020 for rigid docking. 
We sample a rotated ligand endpoint $x_T$ using the Riemannian Score-Based Generative Models (RSGM) scheduler, with the original ligand $x_0$, and sample intermediates as $x_t = m_tx_0 + (1-m_t)x_T + \sigma_t \epsilon$, where $\epsilon \sim N(0, I)$ and $m_t=\frac{t}{T}, \sigma_t=2(m_t - m_t^2)$ using $T=100$ diffusion steps. As the Euclidean baseline, we train an equivariant Fisher score network with $3N$ degrees of freedom to predict the ground-truth pose $\hat{x}_0$. In this setting intermediate perturb ligand coordinates $x_t$ do not resemble ligands due to the linear interpolation and addition of Gaussian noise, while the learning task is to predict the ground-truth ligand coordinate, given the static protein pocket. The output prediction head in BBDM is $3N$, compared to to GSM and RSGM which model $6$ dimensions accounting for global rotation and translation. 

To compare all modeling approaches with respect to the dynamics using the same network architecture, we perform 5 dockings per protein-ligand complex in the CrossDocked test dataset comprising 100 complexes and compute the mean RMSD between ground-truth coordinates and predicted coordinates.

\section{Lie group-induced flow matching modeling}
\label{app:flow_matching}

\begin{wrapfigure}{r}{0.30\textwidth}
    \centering
    \includegraphics[width=0.28\textwidth]{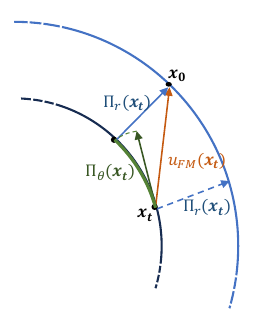}
    \caption{$\so(2)$ (green and blue) vs. $\mathfrak{t}(2)$ (orange) induced flows.}
    \label{fig:flow}
\end{wrapfigure}
We briefly summarize the formalism of flow matching. Given a target distribution 
$p_0(\bx)$ and a vector field $u_t$ generating the distribution $p_t(\bx)$, i.e., if it satisfies $p_t(\bx) = [u_t]_* p_0(\bx)$ where $[u_t]_*$ is the push-forward map, 
the flow matching objective is defined as
\begin{align}
   \cL_{\text{FM}}(\theta) = \E_{t,\bx_t\sim p_t(\bx)}\left| v_{t;\theta}(\bx_t) -u_t(\bx_t)\right|^2.
\end{align}

Marginalizing over samples $\bx_0\sim p_0(\bx)$ we obtain
the conditional flow matching objective
\begin{align}
   \cL_{\text{CFM}}(\theta) = \E_{t,\bx_0\sim p_0(\bx),\bx_t\sim p_t(\bx|\bx_0)}\left| v_{t;\theta}(\bx_t) -u_t(\bx_t|\bx_0)\right|^2.
\end{align}
Now, under the assumptions for learning the generalized score through the objective \ref{eq:score_objective} 
we have that $p_t(\btau(\bx)) = \mathcal{N}(\boldsymbol{\tau} | \bmu(\btau(0), t), \bSigma(t))$, where $\btau(0) = \bx(\btau)(0)$. Then the solution of the SDE from Theorem \ref{th:main_theorem}
\begin{align}
    \bx(t) = \left( \prod_{i=1}^ne^{\tau_i(t) A_i} \right) \bx(0)~, 
\end{align}
is a flow inducing the distribution $p_t(\btau(\bx))$. Thus, 
the vector field that generates the conditional probability path is obtained by differentiating the path above with respect to $t$, yielding
\begin{align}
    u_t(\bx(t)|\bx(0)) = \frac{d}{dt}\bx(t) & = \sum_i \frac{\p \bx(t)}{\partial\tau_i} \frac{\p \tau_i}{\partial t}  \nonumber\\
    &=\sum_i A_i \bx(t) \left(\mu_{t,i}'(\tau_i(0)) + \sigma_{t}'(\tau_i(0)) \eta_i \right)  \nonumber\\
        &= \bPi(\bx(t)) \left[ \bmu_{t}'(\tau_i(0)) + \frac{\sigma_t'(\tau_i(0))}{\sigma_t(\tau_i(0))} \left( \btau(t) - \bmu_t(\btau(0)) \right)\right]~,
\end{align}
where we used the fact that 
\begin{align}
    \btau(t) = \mu_t(\btau(0)) + \sigma_t(\btau(0)) \bbeta~,
\end{align}
where $\bbeta \sim \cN(\boldsymbol{0}, \boldsymbol{1})$.
Thus, we see that the unique vector field that defines the flow \eqref{eq:SDE_solution} is again proportional to the 
fundamental vector field $\bPi(\bx)$ of the Lie algebra $\g$ of $G$. In figure \ref{fig:flow} we illustrate the flow generated by our formalism in the case of $\SO(2)$ in comparison with the traditional flow matching of $T(2)$. The orange path depicts the linear (in Euclidean metric) displacement given by the traditional flow matching, assuming $G=T(2)$. In green and blue we depicted the orbits trajectories resulting from generalized flow matching with $G=SO(2)\times \R_+$.  Although the start and end points are the same, the path is decomposed into transformations along the orbits of the two group factors. This is particularly useful when these correspond to meaningful degrees of freedom in the system. For example, when flowing between conformers of the same molecule, the intermediate states produced by traditional flow matching are often unphysical, as they involve linear interpolation between the Cartesian coordinates of the atoms. However, generalized score matching, following the degrees of freedom given by bond and torsion angles as described in Section \ref{ss:torsion}, would not only yield efficient learning but also produce chemically meaningful intermediate states, as the path is broken down into updates of chemically relevant degrees of freedom.

\end{document}